\documentclass{article} 
\usepackage{iclr2027_conference,times}
\iclrpreprintcopy 

\usepackage{amsmath,amsfonts,bm}

\def\eqref#1{equation~\ref{#1}}

\def\1{\bm{1}}

\DeclareMathAlphabet{\mathsfit}{\encodingdefault}{\sfdefault}{m}{sl}
\SetMathAlphabet{\mathsfit}{bold}{\encodingdefault}{\sfdefault}{bx}{n}

\usepackage{hyperref}
\usepackage{url}

\newcounter{commcounter}
\newcommand{\model}{Top-A}
\newcommand{\modelext}{Topological Attention}
\newcommand{\ie}{i.e., }
\newcommand{\eg}{e.g., }
 
\usepackage{todonotes}
\usepackage{cleveref}
\usepackage{subcaption}
\usepackage[T1]{fontenc}
\usepackage{amsmath,amssymb,amsfonts,amsthm,bm,mathtools}
\usepackage{booktabs,multirow,array,adjustbox,longtable}
\usepackage{graphicx}
\usepackage{float}
\usepackage{microtype}
\usepackage[dvipsnames,svgnames,x11names,table]{xcolor}
\usepackage{enumitem}
\usepackage{tikz}
\usetikzlibrary{arrows.meta,positioning,fit,calc,matrix}
\usepackage{ifthen}
\usepackage{thm-restate}
\usepackage{pifont}
\usepackage{etoc}
\usepackage{mdframed}
\usepackage{wrapfig}
\usepackage[separate-uncertainty=true]{siunitx}
\hypersetup{colorlinks=true,linkcolor=blue!55!black,citecolor=blue!55!black,urlcolor=blue!55!black}
\setlist{nosep,leftmargin=*}

\definecolor{Green}{RGB}{0,128,0}
\definecolor{Red}{RGB}{200,0,0}
\newcommand{\cmark}{\textcolor{Green}{\ding{51}}}
\newcommand{\xmark}{\textcolor{Red}{\ding{55}}}

\definecolor{Top1}{HTML}{0072B2}
\definecolor{Top2}{HTML}{E69F00}
\definecolor{Top3}{HTML}{CC79A7}
\definecolor{ufill}{HTML}{E1D5E7}
\definecolor{uborder}{HTML}{9673A6}
\definecolor{vfill}{HTML}{D5E8D4}
\definecolor{vborder}{HTML}{82B366}

\definecolor{proofbar}{RGB}{225,225,225}
\definecolor{theorembar}{RGB}{128,0,128}
\definecolor{propositionbar}{RGB}{0,100,0}
\definecolor{definitionbar}{RGB}{0,0,200}
\definecolor{lemmabar}{RGB}{255,140,0}
\definecolor{corollarybar}{RGB}{204,153,255}
\definecolor{remarkbar}{RGB}{0,139,139}
\definecolor{assumptionbar}{RGB}{178,34,34}

\usepackage[most]{tcolorbox}
\tcbuselibrary{skins,breakable}
\tcbset{
    mybarstyle/.style={
        enhanced,
        breakable,
        sharp corners,
        frame hidden,      
        interior hidden,   
        boxrule=0pt,
        colframe=#1,
        borderline west={2pt}{0pt}{#1},  
        left=6pt,
        right=0pt,
        top=0pt,
        bottom=0pt,
        before skip=10pt,
        after skip=10pt,
    }
}

\theoremstyle{plain}
\newtheorem{theorem}{Theorem}[section]
\newtheorem{proposition}[theorem]{Proposition}

\theoremstyle{definition}
\newtheorem{definition}[theorem]{Definition}
\theoremstyle{remark}

\newif\ifshowtheorembars
\ificlrfinal\showtheorembarstrue\fi
\ificlrpreprint\showtheorembarstrue\fi
\newcommand{\theorembarstyle}[2]{%
    \addtotheorempreheadhook[#1]{%
        \ifshowtheorembars\begin{tcolorbox}[mybarstyle=#2]\fi}%
    \addtotheorempostfoothook[#1]{%
        \ifshowtheorembars\end{tcolorbox}\fi}%
}
\theorembarstyle{theorem}{theorembar}
\theorembarstyle{proposition}{propositionbar}
\theorembarstyle{lemma}{lemmabar}
\theorembarstyle{corollary}{corollarybar}
\theorembarstyle{definition}{definitionbar}
\theorembarstyle{remark}{remarkbar}
\theorembarstyle{assumption}{assumptionbar}
\theorembarstyle{proof}{proofbar}

\newcommand{\todoresult}[1]{}

\newboolean{showcomments}
\setboolean{showcomments}{true}
\newcommand{\fr}[1]{\textcolor{red}{\textbf{[FR: #1]}}}

\title{ 
Let the Heads Talk:\\Beyond Diagonal Graph Attention
}
\author{%
\textbf{Riccardo Ali}\textsuperscript{\rm 1}\footnotemark[1]\quad
\textbf{Alessio Borgi}\textsuperscript{\rm 1,3}\thanks{%
Equal contribution. Contacts: Riccardo Ali: \href{mailto:rma55@cam.ac.uk}{\nolinkurl{rma55@cam.ac.uk}};
Alessio Borgi: \href{mailto:alessio.borgi@uniroma1.it}{\nolinkurl{alessio.borgi@uniroma1.it}}, \href{mailto:ab3352@cam.ac.uk}{\nolinkurl{ab3352@cam.ac.uk}};
Mario Severino: \href{mailto:mario.severino@phd.unipd.it}{\nolinkurl{mario.severino@phd.unipd.it}};
Alessio Gravina: \href{mailto:alessio.gravina@di.unipi.it}{\nolinkurl{alessio.gravina@di.unipi.it}};
Christopher Irwin: \href{mailto:cli23@cam.ac.uk}{\nolinkurl{cli23@cam.ac.uk}}.}\quad
\textbf{Mario Severino}\textsuperscript{\rm 1,2}\footnotemark[1]\quad
\textbf{Alessio Gravina}\textsuperscript{\rm 4}\\
\textbf{Davide Bacciu}\textsuperscript{\rm 4}\quad
\textbf{Pietro Li\`o}\textsuperscript{\rm 1}\quad
\textbf{Christopher Irwin}\textsuperscript{\rm 1}\footnotemark[1]\\[0.5em]
\small\textsuperscript{\rm 1}Department of Computer Science and Technology, University of Cambridge\\
\small\textsuperscript{\rm 2}Department of Information Engineering, University of Padua\\
\small\textsuperscript{\rm 3}Department of Computer, Control and Management Engineering, Sapienza University of Rome\\
\small\textsuperscript{\rm 4}Department of Computer Science, University of Pisa%
}

\makeatletter
\let\appendixpapertitle\@title
\makeatother

\begin{document}
\maketitle

\etocsettocdepth.toc{none}


\begin{abstract}
Sheaf Neural Networks generalize scalar-weighted message passing by replacing
scalar edge weights with linear transport maps between local feature spaces.
Yet the role of this matrix-valued transport is entangled with the broader
sheaf-diffusion construction. We isolate the transport primitive through
quiver representations and establish a direct connection with multi-head
attention. Treating attention heads as coordinates of a local transport space
reveals that standard multi-head attention implements diagonal edge maps:
along each directed interaction, a source head can contribute only to the
corresponding receiver head. Allowing off-diagonal entries instead enables
edge-conditioned communication across heads before neighborhood aggregation.
We show that this operation cannot, in general, be absorbed into a single
shared linear map applied after aggregation. Building on this characterization,
we introduce \textbf{\modelext{}} (\model{}), a multi-head attention that learns edge-dependent off-diagonal routes while
preserving the original same-head paths and exactly recovering vanilla
attention when the additional routing vanishes. We evaluate \model{} on
relational reasoning, heterogeneous graph learning, and algorithmic reasoning,
including out-of-distribution generalization, with heterophilic node
classification as a contrast setting. The results show that cross-head
transport is most useful when the task benefits from interaction-dependent
transformations, while heterophily alone provides no systematic advantage.
These findings identify edge-conditioned cross-head communication as a distinct
computational primitive of matrix-valued transport. 

\end{abstract}
\section{Introduction}
\label{sec:introduction}

Graph Neural Networks (GNNs) \citep{GNN,NN4G,chebnet,MPNN,GCN} learn by propagating information along graph edges, with most message-passing architectures using scalar weights to modulate neighboring representations. Sheaf Neural Networks (SNNs) \citep{hansen2020sheafneuralnetworks,bodnar2022neural} generalize this paradigm by equipping edges with linear transport operators between local feature spaces, allowing neighboring representations to be transformed rather than merely reweighted. Introduced through cellular sheaf theory \citep{hansen2020toward}, this richer form of local transport was motivated by phenomena such as heterophily \citep{luan2024heterophilicgraphlearninghandbook} and oversmoothing \citep{cai2020note,oono2020graph,rusch2023survey}, and has since inspired a growing family of architectures. Yet, the classical sheaf construction couples several components: restriction maps connect node and edge spaces, their compositions induce transport between neighboring nodes, and the resulting transformations are assembled through the sheaf Laplacian. Consequently, the specific computational capability contributed by richer local transport remains unclear. While the original motivation for SNNs is rooted in sheaf diffusion and topology, recent architectures increasingly learn edge-dependent matrix transformations directly \citep{Copresheaf,ribeiro2026cooperative}. Moreover, recent evidence questions whether gains under heterophily and oversmoothing arise specifically from the learned sheaf structure \citep{caralt2026necessity,fiorini2026benchmarkingsheafneuralnetworks}. This motivates a more fundamental question: \emph{what is actually useful about sheaf transport?} We address it by isolating the local transformation induced by
learned sheaf structure. Rather than asking when SNNs outperform conventional GNNs, we ask which additional form of message passing matrix-valued transport enables and how it differs from standard graph propagation. This leads us from cellular sheaves to quiver representations \citep{Copresheaf}, where directed interactions are represented directly by linear maps. At this level, transport itself can be studied independently of the restriction-map factorization and diffusion dynamics of classical SNNs, while revealing a direct connection with multi-head attention.
If attention heads are treated as coordinates of the local transport space,
multi-head attention induces a diagonal map on each directed interaction:
each source head contributes only to the corresponding receiver head, even
though different heads may assign different weights to the same interaction.
Thus, standard attention spans only the diagonal subfamily of
matrix-valued head-space transport. Off-diagonal entries extend
this mechanism with edge-conditioned communication across
heads before neighborhood aggregation. We show that this edge-specific
mixing cannot, in general, be absorbed into a shared post-aggregation
projection, and therefore constitutes a distinct operation beyond standard
multi-head attention. Motivated by this, we introduce \textbf{\modelext{}} (\model{}), which augments standard attention with edge-conditioned off-diagonal routing across heads while preserving the original same-head paths. When the routing vanishes, \model{} reduces exactly to vanilla attention, providing a controlled way to study when cross-head transport is useful independently of other architectural changes.
We evaluate this inductive bias on algorithmic reasoning with CLRS \citep{velickovic2022clrs}, relational reasoning with STaR \citep{khalid2025systematic}, heterogeneous graph learning \citep{harper2015movielens}, and, as a contrast setting, classical heterophilic node classification \citep{platonov2023critical}. Together, these benchmarks probe edge-specific computation across relational composition, directed heterogeneous interactions, and algorithmic size generalization, while isolating whether heterophily alone is sufficient to make cross-head transport beneficial.

Our main contributions are:
\textbf{(i)} We use quiver representations to disentangle matrix-valued edge transport from the broader sheaf-diffusion construction, establishing a direct connection to multi-head attention.
\textbf{(ii)} We characterize matrix-valued transport as edge-conditioned communication across attention heads and prove that, in general,  it cannot be reduced to a shared post-aggregation projection.
\textbf{(iii)} We introduce \model{}, an architectural extension implementing this transport primitive, and evaluate it across relational, heterogeneous, and algorithmic reasoning, including compositional- and size-generalization regimes.

 \section{Background and Related Work}\label{sec:background}
We start by providing the required background on graph models, SNNs, graph quivers, and graph attention as message passing. A deeper discussion of related work is provided in Appendix~\ref{app:extended_related_work}.
\begin{figure}[h]
    \vspace{-0.4cm}
    \centering
    %
    \begin{subfigure}[b]{0.24\textwidth}
    \centering
    \includegraphics[width=\textwidth]{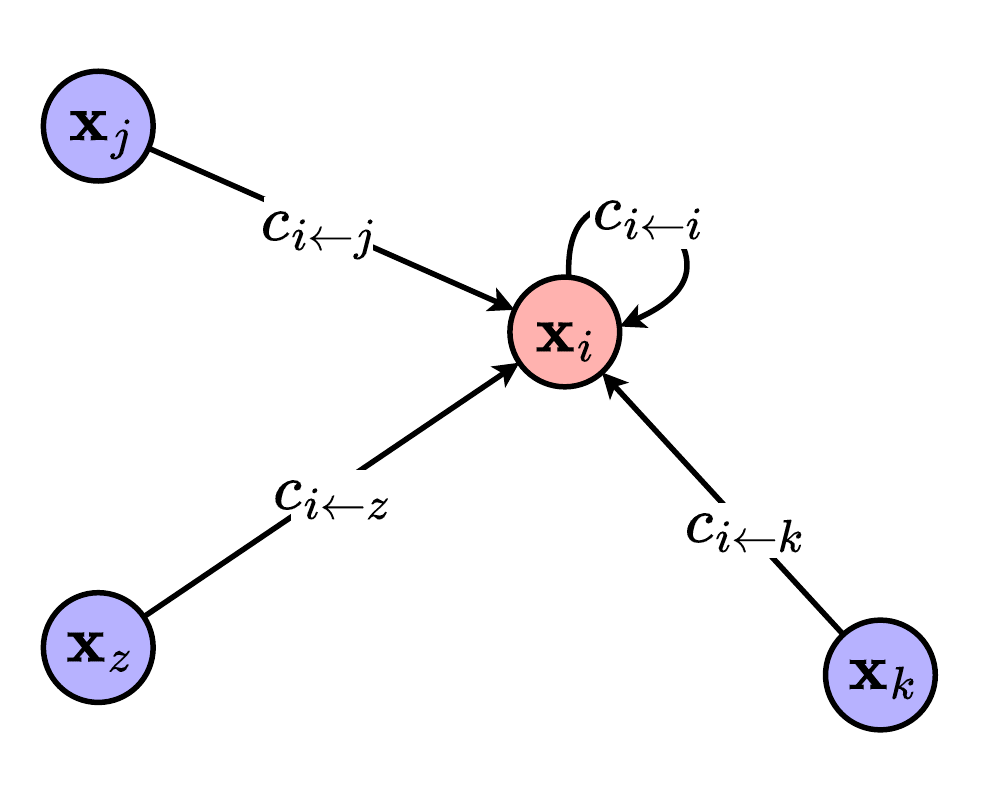}
    \caption{Convolutional\label{fig:conv}}
    \end{subfigure}\hspace{1cm}
    \begin{subfigure}[b]{0.24\textwidth}
    \centering
    \includegraphics[width=\textwidth]{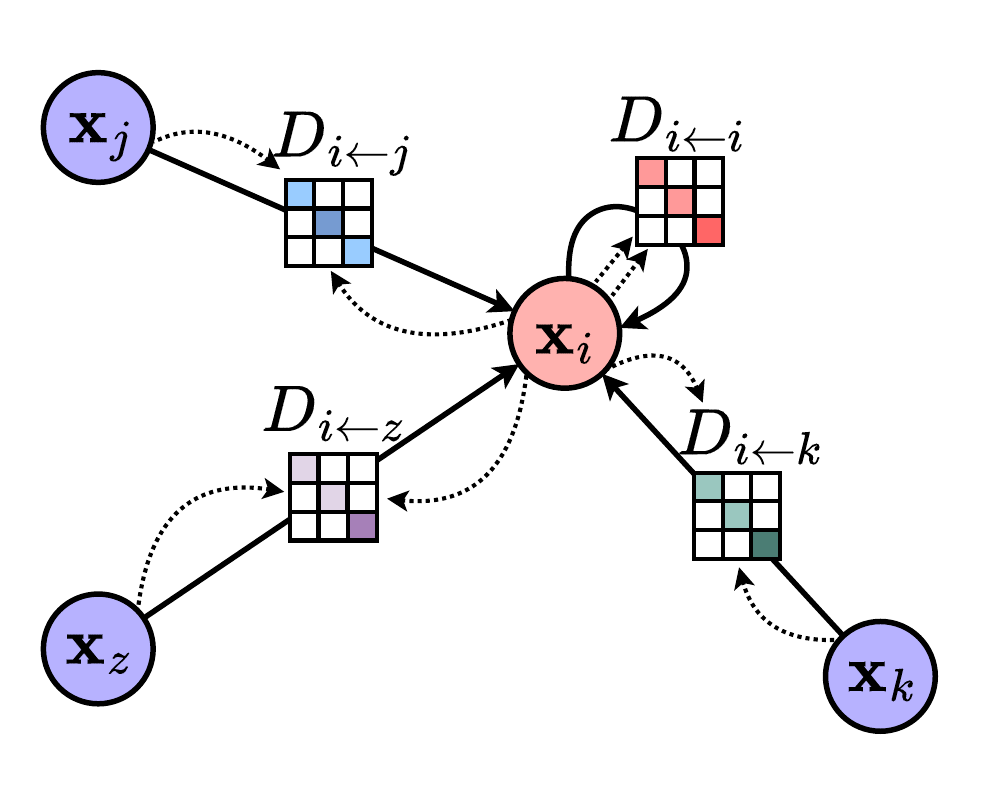}
    \caption{Multi-head Attention\label{fig:attn}}
    \end{subfigure}\hspace{1cm}
    \begin{subfigure}[b]{0.24\textwidth}
    \centering
    \includegraphics[width=\textwidth]{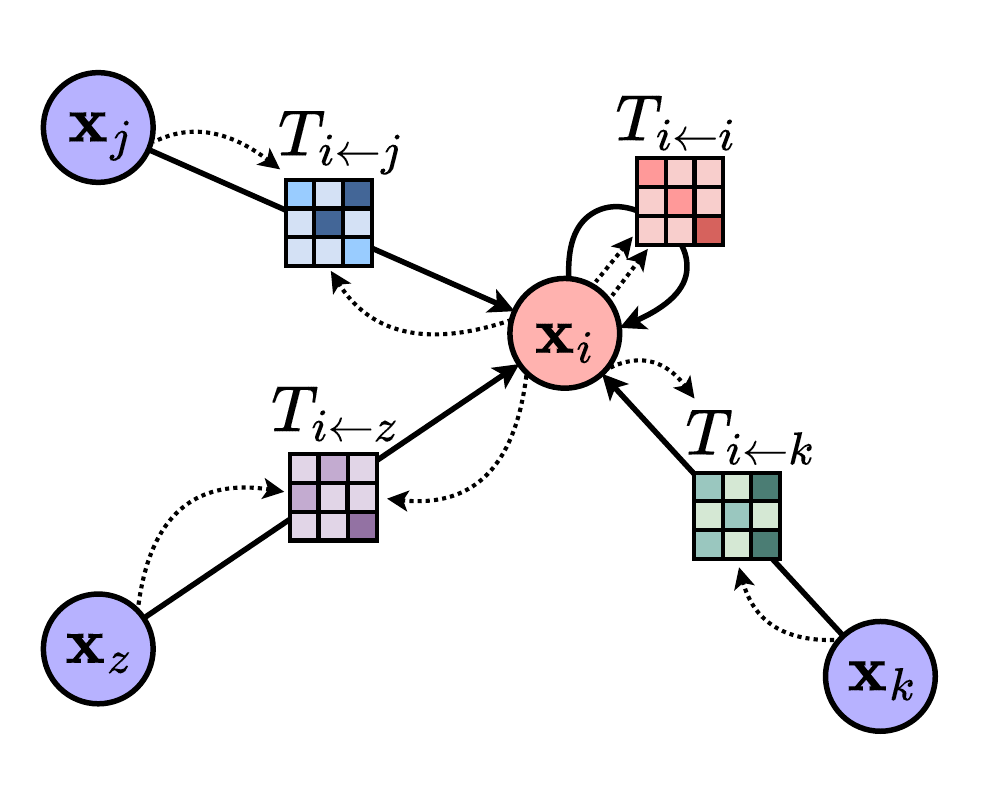}
    \caption{\textbf{\model{} (ours)}\label{fig:top-a_flavour}}
    \end{subfigure}
    
    \caption{
    \textbf{From Scalar Weighting to Matrix-valued Transport.}
    Comparison of 
    three graph neural network paradigms for a directed interaction
    \(j\to i\), with sender \(j\) and receiver \(i\).
    \emph{Convolutional} layers weight incoming features with fixed scalar coefficients
    \(c_{i\leftarrow j}\).
    \emph{Attentional} layers replace these constants with data-dependent coefficients
    \(\alpha_{i\leftarrow j}=a(\mathbf{x}_i,\mathbf{x}_j)\).
    \emph{\modelext{}} replaces the diagonal head-space operator
    \(D_{i\leftarrow j}\) with the matrix-valued transport
    \(T_{i\leftarrow j}
    =D_{i\leftarrow j}(I_H+\lambda\Omega_{i\leftarrow j})\),
    where \(\Omega_{i\leftarrow j}\) contains off-diagonal routes between
    attention heads. Standard multi-head attention is recovered when
    \(\Omega_{i\leftarrow j}=0\).
    }
    \vspace{-0.5cm}
    \label{fig:topological-attention-comparison}
\end{figure}

\textbf{Message Passing Neural Networks.}
Let $\mathcal{G}=(\mathcal{V},\mathcal{E})$ be a graph, where $\mathcal{V}$ is the set of nodes and $\mathcal{E}$ the set of edges. We write a directed interaction from sender node $j$ to receiver node $i$ as $j\to i$, and define the incoming neighborhood of node $i$ as $\mathcal{N}(i)=\{j:(j,i)\in\mathcal{E}\}$. Each node in the graph is associated with a hidden state vector  $\mathbf{x}_i\in\mathbb{R}^{d}$, which provides the node representation. GNNs are functions that learn the complex relational
information of a graph, typically framed through the Message Passing Neural Network (MPNN) framework \citep{MPNN,bronstein2021geometricdeeplearninggrids} where the representation for a node $u$ at layer $\ell$ is computed as:
\begin{equation}
    \mathbf{x}_i^{(\ell)}
    =
    \phi^{(\ell)}
    \left(
        \mathbf{x}_i^{(\ell-1)},
        \bigoplus^{(\ell)}
        \left\{
            \psi^{(\ell)}
            \left(\mathbf{x}_j^{(\ell-1)},\mathbf{x}_i^{(\ell-1)}\right)
            :
            j\in\mathcal{N}(i)
        \right\}
    \right)
\end{equation}
where $\psi^{(\ell)}(\mathbf{x}_j,\mathbf{x}_i)$ computes the message
$
m_{i\leftarrow j}^{(\ell)}$ from sender $j$ to receiver $i$,
$\bigoplus^{(\ell)}$ aggregates the messages in a permutation-invariant way,
and $\phi^{(\ell)}$ updates the node representation. Many widely used GNNs instantiate this framework by associating each neighbor with a scalar coefficient that modulates its contribution to the aggregated representation \citep{GCN,GAT}, as shown in \Cref{fig:conv,fig:attn}.

\textbf{Sheaf Neural Networks.}
SNNs \citep{hansen2020sheafneuralnetworks,bodnar2022neural} enrich graph message passing by associating nodes and edges with local vector spaces, called \emph{stalks}, together with linear maps between them. Given a cellular sheaf $\mathcal{F}$ over $\mathcal{G}$, each node $i\in\mathcal{V}$ and edge $e\in\mathcal{E}$ is assigned a stalk $\mathcal{F}(i)=\mathbb R^d$ and $\mathcal{F}(e)=\mathbb R^d$, respectively, while a restriction map $F_{i\to e}:\mathcal{F}(i)\rightarrow\mathcal{F}(e)$
is associated with every node-edge incidence. For two adjacent nodes $i$ and $j$ sharing an edge $e$, these restriction maps induce the node-to-node linear operator:
\begin{equation}
T_{i\leftarrow j}
=
F_{i\to e}^{\top}F_{j\to e} : \mathcal F(j)\to \mathcal F(e) \to \mathcal F (i)
\label{eq:sheaf_transport}
\end{equation}
which transforms the representation of node $j$ before it is aggregated at node $i$. Neural Sheaf Diffusion (NSD) \citep{bodnar2022neural} builds the propagation dynamics from the resulting sheaf Laplacian, providing a principled form of matrix-valued local communication between neighboring representations.

\textbf{Graph Quivers.} The node-to-node transport induced by SNNs can also be represented directly, without retaining its factorization through the edge stalk. A natural formalism for doing so is a representation of the directed graph quiver \citep{gabriel1972unzerlegbare, derksen2017introduction}. Let the computational graph be viewed as a quiver whose vertices are the nodes of \(\mathcal{G}\) and whose arrows are the directed interactions 
\(e: j\to i\), with source \(j\) and target \(i\). A quiver representation assigns a vector
space \(\mathcal{Q}(i)\) to every vertex and a linear map:
\begin{equation}
    T_{i\leftarrow j}:
    \mathcal{Q}(j)
    \longrightarrow
    \mathcal{Q}(i)
    \label{eq:quiver-arrow}
\end{equation}
to every directed arrow \(j\to i\). The corresponding neighborhood transport takes the form:
\begin{equation}
    \mathbf{x}_i^{(\ell)}
    =
    \sum_{j\in\mathcal{N}(i)}
    T_{i\leftarrow j}^{(\ell)}
    \mathbf{x}_j^{(\ell-1)}
    \label{eq:quiver_mp}
\end{equation}
The sheaf-induced message in \Cref{eq:sheaf_transport} is recovered by choosing:
\begin{equation}
    T_{i\leftarrow j}
    =
    F_{i\to e}^{\top}F_{j\to e}: \mathcal F(j) \to \mathcal F (i)
\end{equation}

The quiver representation therefore isolates the effective directed transport \(T_{i\leftarrow j}\) from the specific restriction-map factorization and diffusion operator that generate it in classical SNNs. Accordingly, we adopt this abstraction throughout the paper: we study the linear map carried directly by each oriented interaction, following the perspective
developed by \cite{Copresheaf}, who cast topological neural
message passing through local maps assigned to quiver arrows. 

\textbf{Graph Attention as Message Passing.}
Attention-based GNNs
\citep{GAT,dwivedi2021generalization,brody2022attentive,graphgps,graphtransformer}
make a sender's contribution depend on its interaction with the receiver.
This includes Graph Attention Networks (GATs)
\citep{GAT,brody2022attentive}, which typically attend over local
neighborhoods, and Graph Transformers (GTs)
\citep{dwivedi2021generalization,graphgps,graphtransformer}, which combine
attention with Transformer-style architectures and structural information.

Omitting layer superscripts, let \(V_j^{(h)}\in\mathbb R^C\) be the value
produced by source head \(h\), and \(s_{i\leftarrow j}^{(h)}\) its score for
the interaction \(j\to i\). Receiver-wise normalization gives:
\begin{equation}
    \alpha_{i\leftarrow j}^{(h)}
    =
    \frac{\exp s_{i\leftarrow j}^{(h)}}
    {\sum_{k\in\mathcal N_{\mathrm{att}}(i)}
    \exp s_{i\leftarrow k}^{(h)}}
    \label{eq:attention-normalization}
\end{equation}
where $\mathcal{N}_{\mathrm{att}}(i)=\mathcal{N}(i)$ for methods such as \cite{GAT} and $\mathcal{N}_{\mathrm{att}}(i)=\mathcal{V}$ for transformer-like architectures, \eg \cite{graphtransformer}. The corresponding head output is:
\begin{equation}
    M_i^{(h)}
    =
    \sum_{j\in\mathcal N_{\mathrm{att}}(i)}
    \alpha_{i\leftarrow j}^{(h)}V_j^{(h)},
    \label{eq:attention-head-update}
\end{equation}
where the scoring rule and attended neighborhood depend on the architecture. This formulation leaves the joint \(H\)-head representation space implicit. The quiver perspective makes it explicit, allowing us to identify the edge map induced by standard
multi-head attention, which we characterize next. 
Appendix~\ref{app:extended_related_work} provides a more in depth comparison with more recent attention mechanisms.

\section{Attention as Diagonal Head-Space Transport}
\label{sec:attention-transport}

The quiver formulation introduced in \Cref{sec:background} describes message
passing through linear maps attached to directed interactions. Multi-head
attention already comes with a natural local space on which such maps can
act: the space indexed by its \(H\) attention heads. We therefore take the
heads as the coordinates of the local transport space, while the
\(C\)-dimensional representation carried by each head is propagated in
parallel. In this scenario, for a sender \(j\), we stack the value representations, $V_j^{(h)}$, introduced in
\Cref{sec:background} as
\(V_j=[(V_j^{(1)})^\top,\ldots,(V_j^{(H)})^\top]^\top
\in\mathbb R^{H\times C}\), where \(V_j^{(h)}\in\mathbb R^C\) denotes the value carried by head \(h\) as before.
For the directed interaction \(j\to i\), the corresponding attention
coefficients define a diagonal matrix of the form:
\begin{equation}
    D_{i\leftarrow j}
    =
    \operatorname{Diag}
    \left(
        \alpha_{i\leftarrow j}^{(1)},
        \ldots,
        \alpha_{i\leftarrow j}^{(H)}
    \right)
    \in
    \mathbb R^{H\times H}
    \label{eq:attention-diagonal-operator}
\end{equation}
This diagonal organization of attention over the head space is illustrated in \Cref{fig:attn}. The \(H\) headwise updates in \Cref{eq:attention-head-update} can be written as:
\begin{equation}
    M_i^{\mathrm{att}}
    =
    \sum_{j\in\mathcal N_{\mathrm{att}}(i)}
    D_{i\leftarrow j}V_j
    \label{eq:attention-matrix-aggregation}
\end{equation}

The form expressed in \Cref{eq:attention-matrix-aggregation} makes the connection with quivers explicit. Once the attention coefficients are computed, they define a linear map on each directed interaction. Taking \(\mathcal Q(i)=\mathbb R^H\) as the local space at each
node, the interaction \(j\to i\) carries the linear map:
\begin{equation}
    T_{i\leftarrow j}^{\mathrm{att}}
    =
    D_{i\leftarrow j}
    :
    \mathbb R^H
    \longrightarrow
    \mathbb R^H
    \label{eq:attention-quiver-map}
\end{equation}
The within-head coordinates are carried along unchanged by this head-space
map. Under the decomposition
\(\mathbb R^{H C}\simeq\mathbb R^H\otimes\mathbb R^C\), the
corresponding operator on the full value representation is therefore
\(D_{i\leftarrow j}\otimes I_C\).

\begin{restatable}[Multi-head attention as diagonal transport]
{proposition}{attentiondiagonal}
\label{prop:attention-diagonal}
Before any shared post-aggregation readout, normalized multi-head attention
is diagonal transport on its native head space. In particular, for every
directed interaction \(j\to i\) and receiver head
\(r\in\{1,\ldots,H\}\):
\begin{equation}
    \left(
        D_{i\leftarrow j}V_j
    \right)^{(r)}
    =
    \alpha_{i\leftarrow j}^{(r)}
    V_j^{(r)}
    \label{eq:attention-diagonal-action}
\end{equation}
where \((\cdot)^{(r)}\) denotes the \(r\)-th head component.
\end{restatable}

A proof of the proposition is given in Appendix~\ref{app:proof-diagonal-attention}. 
Although different heads may assign different weights to the same
interaction, the correspondence between source and receiver coordinates is
fixed: information carried by source head \(r\) is passed only to receiver
head \(r\) along that edge. Standard multi-head attention therefore occupies
the diagonal subfamily of linear maps on the head space. 
Similarly, the multi-head copresheaf attention construction of \citet{Copresheaf} uses edge-dependent linear maps that act independently within the value-channel space of each head (see Appendix~\ref{app:extended_related_work} for further discussion).
\Cref{fig:topological-attention-comparison} summarizes this progression from scalar weighting
to diagonal attention and, ultimately, matrix-valued head-space transport.

\begin{figure*}[t]
    \centering
    \includegraphics[width=0.99\textwidth]{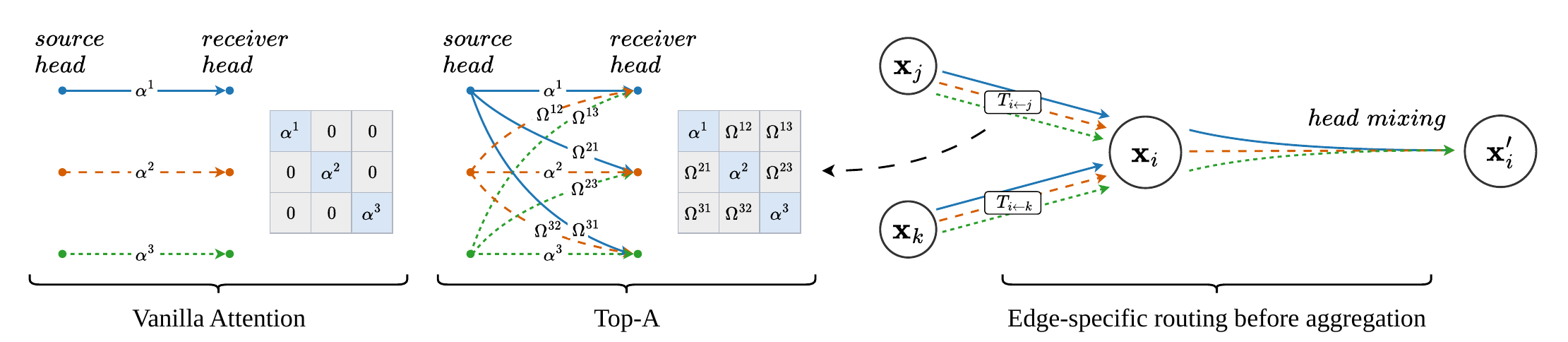}
    \caption{
    \textbf{Vanilla attention versus \model{} Cross-head Routing.}
    Standard multi-head attention restricts information transfer to corresponding source and receiver heads, yielding a diagonal transport structure.
    \model{} generalizes this mechanism through learnable cross-head coefficients $\Omega^{rs}$, enabling source head $s$ to contribute to receiver head $r$.
    The resulting edge-specific transport ${T}_{i \leftarrow j}$ is applied before neighborhood aggregation, followed by head mixing to produce the updated representation $h_i'$.
    }
    \label{fig:topa_overview}
\end{figure*}

\textbf{Beyond diagonal head transport.}
The same representation also makes clear what lies outside this diagonal
family. Consider a general linear map associated with the interaction
\(j\to i\), \(T_{i\leftarrow j}\in\mathbb R^{H\times H}\).
Its action on the sender values is:
\begin{equation}
    \left[
        T_{i\leftarrow j}V_j
    \right]^{(r)}
    =
    \sum_{s=1}^{H}
    T_{i\leftarrow j}^{rs}
    V_j^{(s)}
    \label{eq:general-head-transport-components}
\end{equation}
An off-diagonal entry
\(T_{i\leftarrow j}^{rs}\), with \(r\neq s\), allows information carried
by source head \(s\) to contribute to receiver head \(r\) on that
interaction. Since the map is attached to a directed edge, these
cross-head correspondences may vary across the neighborhood: two senders
connected to the same receiver can induce different transformations before
their contributions are aggregated. On the full value representation, the edge operator becomes \(T_{i\leftarrow j}\otimes I_C\): it mixes information across heads while leaving the \(C\)-dimensional coordinates within each head unchanged. Standard attention is recovered by setting
\(T_{i\leftarrow j}=D_{i\leftarrow j}\), while the off-diagonal entries of
\(T_{i\leftarrow j}\) introduce the additional edge-dependent routes through
which information can move between heads before aggregation, as illustrated in \Cref{fig:topa_overview}.
\Cref{sec:topological-attention} defines a parameterization of these
routes that preserves the original diagonal attention paths.


\textbf{Relation to sheaf transport.}
%
The head-space perspective also sharpens the connection to the sheaf
construction of \Cref{sec:background}. Identifying node stalk coordinates with attention heads, the sheaf-induced map
\(T_{i\leftarrow j}=F_{i\to e}^{\top}F_{j\to e}\) then defines a
matrix-valued transport over the same head space. Its
off-diagonal entries therefore have the same computational role described above: they route information across different head
coordinates. For an undirected edge, however, the two directional transports are coupled through the restriction maps and satisfy
\(T_{j\leftarrow i}=T_{i\leftarrow j}^{\top}\). By contrast, the directed maps studied here need not to satisfy this transpose-reciprocity constraint.
At the level of local operators, classical sheaf transport is therefore an undirected instance of matrix-valued head-space transport, whereas
the quiver formulation assigns independent maps to the two directed interactions. Appendix~\ref{app:sheaf-headspace} formalizes this
connection and proves that symmetric multi-head attention can be realized exactly as diagonal cellular-sheaf transport. Together, these
observations connect standard attention to the diagonal component of sheaf transport and motivate the more general directed, off-diagonal
head-space maps used in \model{}.

\section{\modelext{}}
\label{sec:topological-attention}

\Cref{sec:attention-transport} shows that standard multi-head
attention is restricted to diagonal transport in head space. 
In this section we present \modelext{} (\model{}) that extends this mechanism with edge-dependent communication across different
heads, while retaining the original same-head attention paths. These
cross-head routes are applied independently to each directed interaction
before the incoming messages are aggregated.

\subsection{Edge-conditioned cross-head transport}
\label{sec:topa-transport}

For a directed interaction \(e=(j\to i)\), write
\(D_e:=D_{i\leftarrow j}\), and let
\(\Omega_e:=\Omega_{i\leftarrow j}\in\mathbb R^{H\times H}\)
denote a routing matrix over the head space. We constrain
\(\operatorname{diag}(\Omega_e)=0\)\footnote{Here and throughout the paper, we use \(\operatorname{diag}(A)\) to denote the vector of diagonal entries of a matrix \(A\), and \(\operatorname{Diag}(v)\) for the diagonal matrix whose diagonal is \(v\).} so that the additional branch acts only between different heads. Starting from the diagonal attention operator \(D_e\), we define the \model{} transport associated with \(e\) as \(T_e:=T_{i\leftarrow j}\):
\begin{equation}
    T_e
    =
    D_e
    \left(
        I_H+\lambda\Omega_e
    \right),
    \qquad
    \lambda\geq 0
    \label{eq:topa-transport}
\end{equation}
where \(\lambda\) controls the strength of the cross-head routing. The
resulting operator \(T_e\) replaces \(D_e\) in the edge message.
%
Considering a receiver head \(r\), the corresponding component of the incoming
representation at node \(i\) is
\begin{equation}
    \left[
        M_i^{\model{}}
    \right]^{(r)}
    =
    \left[
        \sum_{j\in\mathcal N_{\mathrm{att}}(i)}
        T_{i\leftarrow j}V_j
    \right]^{(r)}
    =
    \sum_{j\in\mathcal N_{\mathrm{att}}(i)}
    \alpha_{i\leftarrow j}^{(r)}
    \left[
        V_j^{(r)}
        +
        \lambda
        \sum_{\substack{s=1\\s\neq r}}^{H}
        \Omega_{i\leftarrow j}^{rs}
        V_j^{(s)}
    \right].
    \label{eq:topa-components}
\end{equation}
The first term is the usual same-head contribution, while the second aggregates cross-head contributions from source heads \(s\neq r\) through
\(\Omega_{i\leftarrow j}^{rs}\). Since \(D_e\) acts on the left in
\Cref{eq:topa-transport}, \(\alpha_{i\leftarrow j}^{(r)}\) scales the entire
message entering receiver head \(r\), including cross-head contributions.
Attention therefore retains its receiver-side weighting role, while
\(\Omega_e\) determines how source heads are combined before entering that
receiver head. Since \(\operatorname{diag}(\Omega_e)=0\), the routing branch
affects only cross-head interactions, leaving the original attention
coefficients on the diagonal, i.e., \(\operatorname{diag}(T_e)=\operatorname{diag}(D_e)\). When \(\Omega_e=0\), \Cref{eq:topa-transport} reduces exactly to standard multi-head attention. \Cref{fig:topa_overview} illustrates this distinction, contrasting the diagonal head-to-head paths of standard attention with the edge-specific cross-head routes introduced by \model{} before aggregation.

\textbf{Learning the routing matrix.}
\label{sec:topa-routing}
Let \(\xi_e\in\mathbb R^p\) denote a representation of the directed
interaction \(e\), constructed from quantities already available to the
attention layer. 
In \model{}, we compute the routing matrix using a learned map
\(g_\theta:\mathbb R^p\to\mathbb R^{H\times H}\), as
\begin{equation}
    \Omega_e
    =
    \operatorname{off}\!\left(\tanh g_\theta(\xi_e)\right),
    \qquad
    \operatorname{off}(A)
    =
    A-\operatorname{Diag}(\operatorname{diag}(A)).
    \label{eq:topa-route-generator}
\end{equation}
The operator \(\operatorname{off}(\cdot)\) removes the diagonal, ensuring
that the learned routing acts only between different heads, while the
elementwise \(\tanh\) bounds the routing coefficients and allows them to
take either sign. Since \(\xi_e\) is specific to a directed interaction,
different edges can induce different routing patterns: \(j\to i\) and
\(i\to j\) need not share the same map, nor do different senders connected
to the same receiver. Backbone-specific definitions of \(\xi_e\) are given
in Appendix~\ref{app:backbone-integration}.

\subsection{Routing Before Neighborhood Aggregation}
\label{sec:topa-postaggregation}

We next investigate whether a post-aggregation linear projection can reproduce the same cross-head mixing. The key distintion lies in operational order: such a project acts only 
after all incoming messages have been combined, whereas \model{} applies an interaction-specific head-space map to each
incoming message before aggregation. 
To formalize this distinction, we stack sender \(j\)'s \(H\) head values as
\(v_j:=\big[(V_j^{(1)})^\top,\ldots,(V_j^{(H)})^\top\big]^\top
\in\mathbb R^{HC}\).
The corresponding operators on the concatenated representation are
\[
    \overline D_{i\leftarrow j}
    =
    D_{i\leftarrow j}\otimes I_C,
    \qquad
    \overline T_{i\leftarrow j}
    =
    T_{i\leftarrow j}\otimes I_C.
\]

\begin{restatable}[Shared post-aggregation equivalence]
{theorem}{postaggregationequivalence}
\label{thm:postaggregation-equivalence}

Fix a receiver \(i\) and the edge operators
\(\{\overline D_{i\leftarrow j},
\overline T_{i\leftarrow j}\}_{j\in\mathcal N_{\mathrm{att}}(i)}\).
A single post-aggregation map reproduces \model{}'s routed aggregation for every assignment of sender values if and only if some linear map \(A\in\mathbb R^{HC\times HC}\) satisfies
\begin{equation}
    \sum_{j\in\mathcal N_{\mathrm{att}}(i)}
    \overline T_{i\leftarrow j}v_j
    =
    A
    \sum_{j\in\mathcal N_{\mathrm{att}}(i)}
    \overline D_{i\leftarrow j}v_j,
    \label{eq:postaggregation-equivalence}
\end{equation}
if and only if the same map \(A\) reproduces the transport on every incoming edge:
\begin{equation}
    \overline T_{i\leftarrow j}
    =
    A\overline D_{i\leftarrow j}
    \qquad
    \text{for every }
    j\in\mathcal N_{\mathrm{att}}(i).
    \label{eq:topa-common-left-factor}
\end{equation}
\end{restatable}

\Cref{thm:postaggregation-equivalence} gives an exact criterion for replacing edge-specific pre-aggregation routing with a shared linear map after standard attention aggregation, for fixed attention and transport operators. Because this map acts only after all incoming messages have been merged, it must implement the routed transformation of every incoming edge through the same left factor. Importantly, we allow \(A\) to be any receiver-specific linear map, so the comparison is strictly more permissive than the standard output projection used by the attention backbone.

For the transport in \Cref{eq:topa-transport}, ordinary unmasked
softmax attention with finite logits makes
\(D_{i\leftarrow j}\) invertible. Hence,
\Cref{eq:topa-common-left-factor} is equivalent to
\begin{equation}
    A
    =
    \left[
        D_{i\leftarrow j}
        \left(
            I_H+\lambda\Omega_{i\leftarrow j}
        \right)
        D_{i\leftarrow j}^{-1}
    \right]
    \otimes I_C
    \label{eq:topa-common-factor}
\end{equation}
for every incoming sender \(j\). A shared post-aggregation map therefore exists only when all incoming edges induce
the same effective transformation. Because \(\Omega_{i\leftarrow j}\) may vary across directed interaction, \model{} neither enforces nor generally satisfies this condition. The underlying limitation is information loss: aggregation erases the identity of the individual
incoming messages. A shared readout can mix only the combined representation, whereas \model{} transforms each edge separately beforehand. This ordering produces the essential computational advantage of pre-aggregation routing. Appendix~\ref{app:postaggregation-proof} provides the full proof and a two-edge example.




\section{Experiments}\label{sec:exps}



We evaluate the proposed head-space transport across complementary computational settings. In \Cref{sec:STAR}, a
controlled relational reasoning benchmark tests whether edge-conditioned
cross-head transport improves relation classification, including at reasoning depths beyond those seen during training. In \Cref{sec:Het}, a real-world heterogeneous graph task assesses whether communication across head subspaces provides a useful inductive bias for integrating information from nodes and relations with distinct semantics. In \Cref{sec:CLRS}, we look into algorithmic reasoning tasks, investigating which algorithmic procedures benefit from the addedrouting capacity of \model{} and whether these gains extend to larger problem instances. As a contrast setting, Appendix~\ref{app:platonov} evaluates whether cross-head transport also benefits standard heterophilic
node-classification benchmarks. Task details, including
hyperparameter settings and search spaces, are reported in
Appendices~\ref{app:star_description}, \ref{app:movielens}, and
\ref{app:clrs}. 

\textbf{Attention Backbones.}
Because \model{} is agnostic to the underlying attention mechanism, we instantiate it with GATv2 \citep{brody2022attentive} and a Graph Transformer (GT) \citep{graphtransformer}, comparing each backbone with its corresponding \model{} variants. 
In both architectures, the interaction representation \(\xi_e\) in \Cref{eq:topa-route-generator} is constructed from quantities already available to the respective attention mechanism. For GATv2, \(\xi_e\) uses the same directed sender--receiver context as attention computation; for GT, it combines the receiver query and the edge-conditioned sender key. The resulting routing matrix \(\Omega_e\) remains specific to each directed interaction, allowing distinct edges to induce distinct cross-head transformations. The precise backbone-specific constructions are provided in Appendix~\ref{app:backbone-integration}, while initialization and computational
cost are discussed in Appendix~\ref{app:initialization-complexity}.

\subsection{Relational Reasoning on STaR}

\begin{figure}[t]


\centering

\begin{subfigure}[h]{0.49\linewidth}
    \centering
    \includegraphics[width=0.8\linewidth]{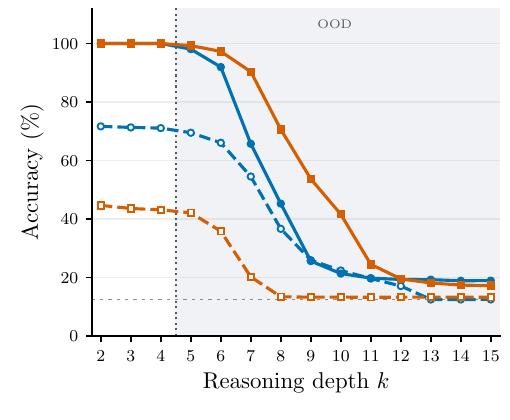}
    \caption{RCC-8}
    \label{fig:star-ood-rcc8}
\end{subfigure}
\hfill
\begin{subfigure}[h]{0.49\linewidth}
    \centering
    \includegraphics[width=0.8\linewidth]{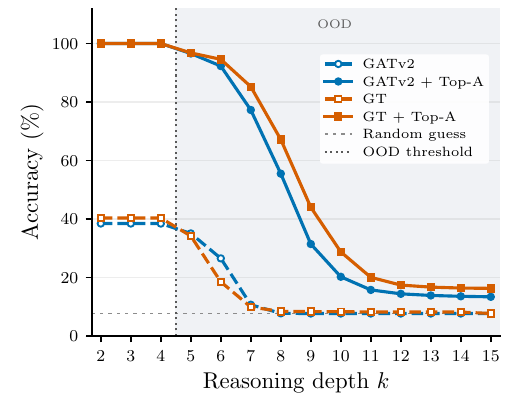}
    \caption{Interval Algebra}
    \label{fig:star-ood-interval}
\end{subfigure}

\caption{\emph{Depth Generalization on STaR.}
For each reasoning depth \(k\), accuracy is averaged over \(b\in\{1,2,3\}\) and three seeds. Models perform 15 recurrent rounds with shared-weights and are trained on \(k\in\{2,3,4\}\); shading denotes unseen depths. Solid lines denote \model{}, and dashed lines the corresponding baselines.}
\label{fig:star-depth-ood}

\end{figure}

\label{sec:STAR}
To evaluate the proposed parameterization in a setting that demands relational composition, we use STaR \citep{khalid2025systematic}, a controlled reasoning benchmark in which the edge relations determine how relational information is composed along a path. STaR tests systematic generalization over synthetic relation graphs in two domains, \ie RCC-8 and Allen's Interval Algebra (IA). Each instance provides an edge-labelled graph and a query node pair, and requires inferring the relation between them. Solving this task entails composing relations along each path and integrating the partial, potentially disjunctive evidence from multiple paths. STaR is therefore well-suited to testing edge-conditioned transport: distinct relations call for distinct transformations of propagated information, while multiple paths require the resulting relational states to be combined.

All evaluated models are trained on graphs with path lengths
\(k\in\{2,3,4\}\) and numbers of paths \(b\in\{1,2,3\}\), then evaluated on longer reasoning chains to measure generalization beyond the depths seen
during training. Within each backbone family, the baseline and \model{}
variant use identical input representations, backbone architectures, prediction heads, and optimization procedures. Appendix~\ref{app:star_description} provides further details.

\noindent\textbf{Results.}
\Cref{tab:star-results} show that \model{} improves mean accuracy for both attention backbones across both relational domains, supporting the value of the proposed head-space parameterization for relational
reasoning. Standard relation-aware attention can condition its scalar coefficients on the edge label; \model{} additionally lets each labelled interaction induce off-diagonal routing across head subspaces before aggregation. This mechanism can support relation-specific transformations and the integration of evidence propagated along multiple paths. 
\Cref{fig:star-depth-ood} plots accuracy against reasoning depth \(k\), 
showing that \model{}'s gains extend across multiple unseen depths.
On both RCC-8 and IA, the \model{} variants outperform their corresponding attention baselines within the training-depth range \(k\in\{2,3,4\}\) and across a substantial range of unseen depths, although the margin narrows on the most challenging configurations. Since \(b\) remains within its training range, this analysis isolates increasing reasoning depth from growth in the number of paths.

\begin{table*}[t]
\centering



\begin{minipage}[t]{0.52\textwidth}
\centering
\caption{\emph{STaR Relation-Classification Accuracy}. We evaluate over 3 seeds, reporting $\pm$ std $\times 100$. Bold indicates the best-performing models.}
\label{tab:star-results}
\small
\resizebox{\linewidth}{!}{%
\begin{tabular}{llcc}
\toprule
Architecture & Variant & RCC-8 & Interval Algebra \\
\midrule
GATv2 & -- & $39.79 \pm 23.83$ & $18.05 \pm 17.94$ \\
 & Top-A & $\boldsymbol{51.79 \pm 7.14}$ & $\boldsymbol{52.44 \pm 1.22}$ \\
\midrule
GT & -- & $23.34 \pm 17.68$ & $18.03 \pm 15.69$ \\
 & Top-A & $\boldsymbol{60.24 \pm 10.21}$ & $\boldsymbol{57.17 \pm 2.54}$ \\
\bottomrule
\end{tabular}%
}

\end{minipage}
\hfill
\begin{minipage}[t]{0.44\textwidth}



\centering
\caption{\emph{Link Prediction on MovieLens-100K.} We evaluate over 10 seeds,
reporting mean $\pm$ std $\times 100$.}
\label{tab:heterogeneous}

\small
\resizebox{\linewidth}{!}{%
\begin{tabular}{llcc}
\toprule
Model & Variant & AUROC & AUPR \\
\midrule
GATv2 & --
      & $88.70 \pm 0.83$
      & $86.48 \pm 1.07$ \\

 & Top-A
      & $\mathbf{90.46 \pm 0.42}$
      & $\mathbf{88.71 \pm 0.51}$ \\
\midrule
GT & --
   & $81.26 \pm 3.67$
   & $76.24 \pm 5.21$ \\

 & Top-A
   & $\mathbf{90.20 \pm 0.85}$
   & $\mathbf{88.38 \pm 0.99}$ \\
\bottomrule
\end{tabular}%
}
\end{minipage}

\end{table*}

\subsection{Relation-Conditioned Routing on Heterogeneous Graphs}
\label{sec:Het}
After establishing the benefits of edge-conditioned cross-head transport in STaR's
controlled relational setting, we evaluate the same
parameterization on a real-world heterogeneous graph task. We use
MovieLens-100K \citep{harper2015movielens}, where users and movies forms semantically distinct node types connected by directed relations. This setting complents STaR by testing whether communication across latent head subspaces helps integrate information with different node and relation semantics. We formulate MovieLens-100K as a binary link-prediction task and compare the same two attention backbones with their corresponding \model{} variants. Within each backbone family, all models use identical heterogeneous input encoders, backbone, decoder, and per-run training protocol. The \model{} variants additionally tune the transport scale \(\lambda\), with the search space and selected configurations reported in Appendix~\ref{app:movielens}.


\noindent\textbf{Results.}
As shown in \Cref{tab:heterogeneous}, \model{} improves mean AUROC and AUPR for both attention backbones across ten evaluation seeds. For GATv2, cross-head transport yields average gains of approximately \(2.0\) points across the two metrics; for GT, the gains increase to \(10.5\) points on average. GT+\model{} closes the performance gap with GATv2 and performs comparably to GATv2+\model{}, while reducing the variability of the original GT backbone. The transport analysis in Appendix~\ref{app:movielens}
(\Cref{fig:movielens-asym}) further reveals distinct effective head-space
transport for the \emph{rates} and \emph{rated-by} directions.
This directional asymmetry indicates that the two relations induce effective head-space transformations.
Detailed transport and optimization analyses are reported in Appendix~\ref{app:movielens}.


\subsection{CLRS Benchmark}
\label{sec:CLRS}
Having evaluated cross-head transport in controlled relational and real-world heterogeneous settings, we next investigate which computational procedures benefit from \model{}'s additional routing capacity. We use the CLRS Algorithmic Reasoning Benchmark \citep{velickovic2022clrs}, which comprises 30 classical algorithms spanning sorting, searching, dynamic programming, graph algorithms, string processing, and computational geometry. This diversity makes CLRS well suited to probing architectural inductive biases, as different algorithms require substantially different forms of information transport. Rather than expecting cross-head transport to provide uniform gains across all tasks, we use CLRS to identify when \model{}'s additional flexibility aligns with the computation required by the target algorithm. Following the standard benchmark protocol, models are trained on instances of size $n=16$ and evaluated out-of-distribution at $n=64$, testing whether the learned routing supports reusable computational procedures beyond the training scale.
Within each backbone family, all models share the same input encoders, processor architecture, decoder, and training protocol. Additional details are provided in Appendix~\ref{app:clrs}.

\noindent\textbf{Results.}
As shown in \Cref{tab:clrs-families}, the gains from cross-head transport are strongly task dependent; Appendix~\ref{app:clrs}, \Cref{tab:clrs-tuned}, reports the full results. 
The largest gains occur on sorting and searching, where GATv2+\model{} improves over GATv2 by more than $29$ and $12$ points, respectively. This trend is also evident with the GT, for which \model{} improves searching algorithms by more than $24$ points. 
By contrast, families such as greedy algorithms and dynamic programming show little or no improvement. Averaged over all 30 tasks, \model{} improves GATv2 by $6$ points and GT by $8.5$ points. This heterogeneous pattern supports our hypothesis that cross-head transport provides a task-dependent inductive bias rather than a uniform advantage across algorithms. 
The pronounced gain on sorting and searching suggest that these procedures are particularly well matched to flexible information exchange across latent head subspaces during the execution of the learned algorithm.

\providecommand{\model}{\textsc{OffDiag}}
\begin{table*}[h]
\centering
\small
\setlength{\tabcolsep}{4pt}
\caption{CLRS-30 test scores (official CLRS score; \%, $\uparrow$) after separate hyperparameter optimization, grouped by the benchmark's eight algorithm families. Entries are mean $\pm$ sample SD across 20 seeds, averaging tasks equally within each seed. Overall averages all 30 tasks equally. $\Delta$ is the improvement from adding \model{} in percentage points, computed before rounding. Bold marks the higher displayed mean within each architecture.}
\label{tab:clrs-families}
\resizebox{0.95\linewidth}{!}{%
\begin{tabular}{lcrrrrrr}
\toprule
Family & Tasks & GATv2 & GATv2 + \model{} & $\Delta$ & GT & GT + \model{} & $\Delta$ \\
\midrule
Sorting & 4 & $25.18 \pm 7.62$ & $\mathbf{54.42} \pm 6.77$ & $+29.24$ & $23.81 \pm 9.40$ & $\mathbf{54.17} \pm 8.43$ & $+30.36$ \\
Searching & 3 & $42.31 \pm 4.46$ & $\mathbf{54.72} \pm 4.01$ & $+12.41$ & $31.92 \pm 1.37$ & $\mathbf{56.24} \pm 2.62$ & $+24.32$ \\
Divide and conquer & 1 & $58.10 \pm 4.78$ & $\mathbf{63.61} \pm 3.34$ & $+5.51$ & $57.19 \pm 4.66$ & $\mathbf{62.91} \pm 3.73$ & $+5.72$ \\
Greedy & 2 & $\mathbf{74.22} \pm 2.56$ & $73.92 \pm 2.80$ & $-0.30$ & $\mathbf{74.98} \pm 2.20$ & $73.94 \pm 2.83$ & $-1.03$ \\
Dynamic programming & 3 & $73.54 \pm 4.00$ & $\mathbf{73.71} \pm 5.92$ & $+0.17$ & $\mathbf{76.34} \pm 4.72$ & $70.02 \pm 7.28$ & $-6.32$ \\
Graphs & 12 & $60.68 \pm 1.54$ & $\mathbf{61.86} \pm 2.04$ & $+1.18$ & $54.51 \pm 1.62$ & $\mathbf{60.51} \pm 1.84$ & $+5.99$ \\
Strings & 2 & $1.39 \pm 0.85$ & $\mathbf{1.54} \pm 0.78$ & $+0.15$ & $1.40 \pm 0.79$ & $\mathbf{1.92} \pm 0.83$ & $+0.53$ \\
Geometry & 3 & $70.26 \pm 1.94$ & $\mathbf{72.63} \pm 2.40$ & $+2.36$ & $70.04 \pm 3.44$ & $\mathbf{71.59} \pm 4.17$ & $+1.55$ \\
\midrule
Overall & 30 & $53.22 \pm 1.26$ & $\mathbf{59.26} \pm 1.48$ & $+6.04$ & $49.81 \pm 1.73$ & $\mathbf{58.36} \pm 1.20$ & $+8.56$ \\
\bottomrule
\end{tabular}%
}
\end{table*}

\section{Conclusion}\label{sec:conclusion}

We revisited matrix-valued transport as a principled generalization of multi-head attention beyond its diagonal structure in head space. By allowing off-diagonal maps, edges can coordinate information flow across heads before aggregation. This perspective led to \model{}, a direct realization of the resulting transport primitive that retains standard attention as a special case. Across relational, heterogeneous, and algorithmic reasoning tasks, this richer routing mechanism consistently provides a useful inductive bias. More broadly, our results establish a direct connection between sheaf transport and attention through a simple yet powerful principle: edges should determine not only how much information is transmitted, but also how it is routed across heads.

\bibliography{iclr2027_conference}
\bibliographystyle{iclr2027_conference}

\clearpage
\appendix
%

\newcommand{\appendixsummary}{%
  \etocsettocdepth.toc{subsubsection}%
  \begin{center}
    {\LARGE\bfseries\appendixpapertitle\par}
    \vspace{0.4em}
    {\Large\bfseries Supplementary Material\par}
  \end{center}
  \begingroup
    \etocsetnexttocdepth{subsubsection}%
    \etocsettocstyle{\section*{Appendix}}{}%
    \tableofcontents
  \endgroup
  \clearpage
}

\providecolor{takeawayblue}{RGB}{48,98,168}
\providecolor{takeawaygray}{RGB}{247,247,247}

\ifcsname takeawaybox\endcsname\else
\newtcolorbox{takeawaybox}{
    enhanced,
    colback=takeawaygray,
    colframe=takeawaygray,
    boxrule=0pt,
    borderline west={2.5pt}{0pt}{takeawayblue},
    sharp corners,
    left=10pt,
    right=10pt,
    top=7pt,
    bottom=7pt,
    before skip=8pt,
    after skip=8pt
}
\fi

\providecolor{Green}{RGB}{0,128,0}
\providecolor{Red}{RGB}{200,0,0}
\providecommand{\cmark}{\textcolor{Green}{\ding{51}}}
\providecommand{\xmark}{\textcolor{Red}{\ding{55}}}

\providecolor{Top1}{HTML}{0072B2}
\providecolor{Top2}{HTML}{E69F00}
\providecolor{Top3}{HTML}{CC79A7}
\providecolor{ufill}{HTML}{E1D5E7}
\providecolor{uborder}{HTML}{9673A6}
\providecolor{vfill}{HTML}{D5E8D4}
\providecolor{vborder}{HTML}{82B366}

\sisetup{
    mode=text,
    table-alignment-mode=none,
    reset-text-family=false,
    reset-text-series=false,
    reset-text-shape=false,
    table-number-alignment=center,
    table-align-comparator=false,
    table-align-text-pre=false,
    table-align-text-post=false,
    round-mode=uncertainty,
    round-precision=3
}

\providecolor{proofbar}{RGB}{225,225,225}
\providecolor{theorembar}{RGB}{128,0,128}
\providecolor{propositionbar}{RGB}{0,100,0}
\providecolor{definitionbar}{RGB}{0,0,200}
\providecolor{lemmabar}{RGB}{255,140,0}
\providecolor{corollarybar}{RGB}{204,153,255}
\providecolor{remarkbar}{RGB}{0,139,139}
\providecolor{assumptionbar}{RGB}{178,34,34}

%
%

\newenvironment{restatedstatementbox}[1]{%
  \ifshowtheorembars
    \begin{tcolorbox}[mybarstyle=#1]%
  \else
    \par\addvspace{\topsep}%
  \fi
}{%
  \ifshowtheorembars
    \end{tcolorbox}%
  \else
    \par\addvspace{\topsep}%
  \fi
}

\newenvironment{restatedtheorem}[2][]{%
  \begin{restatedstatementbox}{theorembar}%
  \noindent\textbf{Theorem~\ref{#2}}%
  \if\relax\detokenize{#1}\relax
    \textbf{.}%
  \else
    \textbf{ (#1).}%
  \fi
  \itshape\ignorespaces
}{%
  \end{restatedstatementbox}
}

\newenvironment{restatedproposition}[2][]{%
  \begin{restatedstatementbox}{propositionbar}%
  \noindent\textbf{Proposition~\ref{#2}}%
  \if\relax\detokenize{#1}\relax
    \textbf{.}%
  \else
    \textbf{ (#1).}%
  \fi
  \itshape\ignorespaces
}{%
  \end{restatedstatementbox}
}

\newenvironment{restatedlemma}[2][]{%
  \begin{restatedstatementbox}{lemmabar}%
  \noindent\textbf{Lemma~\ref{#2}}%
  \if\relax\detokenize{#1}\relax
    \textbf{.}%
  \else
    \textbf{ (#1).}%
  \fi
  \itshape\ignorespaces
}{%
  \end{restatedstatementbox}
}

\newenvironment{restatedcorollary}[2][]{%
  \begin{restatedstatementbox}{corollarybar}%
  \noindent\textbf{Corollary~\ref{#2}}%
  \if\relax\detokenize{#1}\relax
    \textbf{.}%
  \else
    \textbf{ (#1).}%
  \fi
  \itshape\ignorespaces
}{%
  \end{restatedstatementbox}
}

\providecommand*{\theoremautorefname}{Theorem}
\providecommand*{\lemmaautorefname}{Lemma}
\providecommand*{\propositionautorefname}{Proposition}
\providecommand*{\definitionautorefname}{Definition}
\providecommand*{\corollaryautorefname}{Corollary}
\providecommand*{\remarkautorefname}{Remark}
\providecommand*{\exampleautorefname}{Example}
\providecommand*{\assumptionautorefname}{Assumption}

\makeatletter
\@ifundefined{ifshowcomments}{\newboolean{showcomments}}{}
\makeatother
\setboolean{showcomments}{true}
\providecommand{\fr}[1]{\textcolor{red}{\textbf{[FR: #1]}}}

\appendixsummary
\section{Extended Related Work}
\label{app:extended_related_work}

This section expands on the related work introduced in \Cref{sec:background}, providing a more detailed discussion of the literature most closely connected to our contribution.

\paragraph{Relation to copresheaf attention.}
\citet{Copresheaf} formulate neural message passing through copresheaves,
associating local representation spaces with nodes and linear maps with
directed interactions. This directed-map perspective is closely related to
the level of abstraction adopted in our quiver formulation. In their
copresheaf attention construction, however, transport acts within each
attention head: for an interaction \(j\to i\), a learned map
\(\rho_{i\leftarrow j}^{(h)}\in\mathbb{R}^{C\times C}\) transforms the value
carried by head \(h\) before attention-weighted aggregation, with the maps
evaluated independently across heads. On the full representation
\(\mathbb{R}^{HC}\simeq\mathbb{R}^{H}\otimes\mathbb{R}^{C}\), this gives a
block-diagonal operator of the form
\(\rho_{i\leftarrow j}^{(1)}\oplus\cdots\oplus
\rho_{i\leftarrow j}^{(H)}\).
\model{} instead identifies the head space itself as the transport space and
associates each directed interaction with an \(H\times H\) map. Its action on
the full representation is \(T_{i\leftarrow j}\otimes I_C\): information is
mixed across heads while the within-head coordinates are carried unchanged.
The two constructions therefore place matrix-valued transport on different
factors of the head--channel representation: copresheaf attention acts
within heads, whereas \model{} acts between them. This comparison concerns
the specific attention construction of \citet{Copresheaf}; the general
copresheaf formalism is substantially broader.

\paragraph{Relation to head-composition methods.}
To position \model{} with respect to existing forms of cross-head communication, we distinguish our matrix-valued transport view from methods that compose attention heads directly. Importantly, we note that our goal is not to introduce an alternative head-composition mechanism, but to study the computational role of matrix-valued transport through its connection with multi-head attention.

Talking-Heads Attention \citep{shazeer2020talking} introduces linear projections across the head dimension of the attention logits before softmax and of the normalized attention weights after softmax, while DCMHA \citep{xiao2024improving} extends this principle through input-dependent composition of attention scores and weights. Like \model{}, these methods enable communication across heads, but they do so at a different stage of the attention computation. Specifically, Talking-Heads Attention and DCMHA operate on the scalar quantities that determine how each value head weights its inputs, without directly transferring value representations between heads. By contrast, \model{} acts on the value-head representation associated with each directed interaction,
\(\bigl[T_{i\leftarrow j}V_j\bigr]^{(r)}
=\sum_{s=1}^{H}T_{i\leftarrow j}^{rs}V_j^{(s)}\),
allowing the value carried by source head \(s\) to contribute directly to receiver head \(r\) before neighborhood aggregation. The parameterization \(T_{i\leftarrow j}
=D_{i\leftarrow j}(I_H+\lambda\Omega_{i\leftarrow j})\), with
\(\operatorname{diag}(\Omega_{i\leftarrow j})=0\),
retains the original same-head attention paths while adding interaction-dependent off-diagonal routes between value heads, and recovers standard multi-head attention exactly when \(\Omega_{i\leftarrow j}=0\). From this perspective, Talking-Heads Attention and DCMHA compose attention patterns across heads, whereas \model{} provides a direct realization of edge-conditioned transport over the value-head space.

\section{Theoretical and Architectural Details}
\label{app:theory}

\subsection{From Cellular Sheaves to Symmetric Attention}
\label{app:sheaf-headspace}
The main text studies the transport map
\(T_{i\leftarrow j}\) carried by each directed interaction
\(j\to i\), abstracting away from the restriction-map factorization and
diffusion operator of a cellular sheaf. We now return to the cellular-sheaf
setting and show that the diagonal edge transport induced by a symmetric
form of multi-head attention admits an exact realization as sheaf transport.
This establishes a direct connection between diagonal head-space attention
and the local transport structure underlying Neural Sheaf Diffusion
\citep{hansen2020sheafneuralnetworks,bodnar2022neural}.

\begin{definition}[Symmetric multi-head attention]
A multi-head attention mechanism with \(H\) heads is \emph{symmetric} if,
for every undirected edge \(e=\{i,j\}\) and every head
\(h\in\{1,\ldots,H\}\),
\[
    \alpha_{i\leftarrow j}^{(h)}
    =
    \alpha_{j\leftarrow i}^{(h)}
    \geq 0 .
\]
The corresponding head-space operator is
\[
    D_{i\leftarrow j}
    =
    \operatorname{Diag}
    \left(
        \alpha_{i\leftarrow j}^{(1)},
        \ldots,
        \alpha_{i\leftarrow j}^{(H)}
    \right)
    \in\mathbb{R}^{H\times H}.
\]
By symmetry,
\(D_{i\leftarrow j}=D_{j\leftarrow i}\).
\end{definition}

The following proposition shows that every such diagonal edge operator can
be realized by cellular-sheaf restriction maps.

\begin{proposition}[Symmetric attention as sheaf transport]
Consider node and edge stalks of dimension \(H\),
\(\mathcal F(i)=\mathcal F(j)=\mathcal F(e)=\mathbb R^H\).
Given a symmetric multi-head attention mechanism, there exist restriction
maps such that the induced sheaf transport
\[
    T_{i\leftarrow j}
    =
    F_{i\to e}^{\top}F_{j\to e}
\]
satisfies
\[
    T_{i\leftarrow j}
    =
    D_{i\leftarrow j}
\]
for every edge \(e=\{i,j\}\).
\end{proposition}

\begin{proof}
Since \(D_{i\leftarrow j}\) is diagonal with nonnegative entries, its
principal square root is well defined. For each edge \(e=\{i,j\}\), choose
\[
    F_{i\to e}
    =
    F_{j\to e}
    =
    D_{i\leftarrow j}^{1/2}.
\]
Then
\[
    F_{i\to e}^{\top}F_{j\to e}
    =
    \left(D_{i\leftarrow j}^{1/2}\right)^{\top}
    D_{i\leftarrow j}^{1/2}
    =
    D_{i\leftarrow j},
\]
where the final equality follows because
\(D_{i\leftarrow j}^{1/2}\) is diagonal.
The same construction gives
\(T_{j\leftarrow i}=D_{j\leftarrow i}=D_{i\leftarrow j}\)
in the reverse direction.
\end{proof}

Thus, symmetric multi-head attention induces a diagonal head-space transport
that admits an exact realization within the cellular-sheaf framework.

\color{black}

\subsection{Proof of \Cref{prop:attention-diagonal}}
\label{app:proof-diagonal-attention}

\attentiondiagonal*

\begin{proof}
Once the attention coefficients for the interaction \(j\to i\) have been
computed, they define the diagonal matrix
\[
    D_{i\leftarrow j}
    =
    \operatorname{Diag}
    \left(
        \alpha_{i\leftarrow j}^{(1)},
        \ldots,
        \alpha_{i\leftarrow j}^{(H)}
    \right).
\]
The sender values are arranged as
\[
    V_j
    =
    \begin{bmatrix}
        (V_j^{(1)})^\top\\
        \vdots\\
        (V_j^{(H)})^\top
    \end{bmatrix},
\]
so multiplying by \(D_{i\leftarrow j}\) simply scales each head by its own
attention coefficient:
\[
    D_{i\leftarrow j}V_j
    =
    \begin{bmatrix}
        \alpha_{i\leftarrow j}^{(1)}(V_j^{(1)})^\top\\
        \vdots\\
        \alpha_{i\leftarrow j}^{(H)}(V_j^{(H)})^\top
    \end{bmatrix}.
\]
Therefore, for any receiver head \(r\),
\[
    \left(D_{i\leftarrow j}V_j\right)^{(r)}
    =
    \alpha_{i\leftarrow j}^{(r)}V_j^{(r)}.
\]
Because the off-diagonal entries of \(D_{i\leftarrow j}\) are zero, a value
from source head \(s\neq r\) cannot contribute to receiver head \(r\).
The edge map induced by standard multi-head attention is therefore diagonal
in head space.
\end{proof}

\subsection{Shared Post-Aggregation Equivalence}
\label{app:postaggregation-proof}

This section proves \Cref{thm:postaggregation-equivalence} and gives a
two-edge example illustrating why edge-specific routing before aggregation
cannot, in general, be replaced by a single linear map applied afterwards.

\subsubsection{Proof of \Cref{thm:postaggregation-equivalence}}

Recall that
\(\overline D_{i\leftarrow j}=D_{i\leftarrow j}\otimes I_C\) and
\(\overline T_{i\leftarrow j}=T_{i\leftarrow j}\otimes I_C\).

\postaggregationequivalence*

\begin{proof}
Fix a receiver \(i\), the corresponding edge operators, and write
\(\mathcal N_i=\mathcal N_{\mathrm{att}}(i)\).

Suppose first that a linear map \(A\) satisfies
\[
    \sum_{j\in\mathcal N_i}
    \overline T_{i\leftarrow j}v_j
    =
    A
    \sum_{j\in\mathcal N_i}
    \overline D_{i\leftarrow j}v_j
\]
for every choice of the sender values. For any
\(j\in\mathcal N_i\), set all sender values except \(v_j\) to zero. Since
the edge operators are fixed, this gives
\[
    \overline T_{i\leftarrow j}v_j
    =
    A\overline D_{i\leftarrow j}v_j
\]
for every \(v_j\), which implies
\[
    \overline T_{i\leftarrow j}
    =
    A\overline D_{i\leftarrow j}.
\]
Since \(j\) was arbitrary, this holds for every incoming edge.

Conversely, if
\[
    \overline T_{i\leftarrow j}
    =
    A\overline D_{i\leftarrow j}
    \qquad
    \text{for every }j\in\mathcal N_i,
\]
then, for any collection of sender values,
\[
    \sum_{j\in\mathcal N_i}
    \overline T_{i\leftarrow j}v_j
    =
    \sum_{j\in\mathcal N_i}
    A\overline D_{i\leftarrow j}v_j
    =
    A
    \sum_{j\in\mathcal N_i}
    \overline D_{i\leftarrow j}v_j.
\]
This proves both directions.
\end{proof}

The theorem shows that, conditional on the edge operators induced for a
receiver, edge-specific routing can be moved after aggregation only if the
same linear map \(A\) reproduces the routed transformation of every incoming
edge. We allow \(A\) to be an arbitrary linear map for the receiver \(i\),
making this comparison more permissive than the usual output projection of
the attention backbone.

For the \model{} transport in \Cref{eq:topa-transport}, if
\(D_{i\leftarrow j}\) is invertible, the condition becomes
\[
    A
    =
    \left[
        D_{i\leftarrow j}
        \left(I_H+\lambda\Omega_{i\leftarrow j}\right)
        D_{i\leftarrow j}^{-1}
    \right]
    \otimes I_C
\]
for every incoming sender \(j\). A single post-aggregation map therefore
exists only if all incoming edges induce the same effective transformation.
Since \(\Omega_{i\leftarrow j}\) may vary with the directed interaction,
\model{} does not enforce this condition, and it need not hold in general.

\paragraph{Remark.}
\Cref{thm:postaggregation-equivalence} itself does not require
\(D_{i\leftarrow j}\) to be invertible. Invertibility is used only to express
the condition explicitly as \Cref{eq:topa-common-factor}. For
ordinary unmasked softmax attention with finite logits, all attention
coefficients are positive, so \(D_{i\leftarrow j}\) is invertible for every
attended interaction.

\subsubsection{A two-edge example}
\label{app:postaggregation-example}

A simple example makes the distinction concrete. Consider a receiver \(i\)
with two incoming senders \(j_1\) and \(j_2\), two attention heads
(\(H=2\)), and one value coordinate per head (\(C=1\)). Suppose both heads
assign equal attention to the two incoming edges, so that
\[
    D_{i\leftarrow j_1}
    =
    D_{i\leftarrow j_2}
    =
    \frac{1}{2}I_2.
\]
Set \(\lambda=1\), choose some \(\rho\in(0,1)\), and let the two edges use
different routing matrices,
\[
    \Omega_{i\leftarrow j_1}
    =
    \begin{bmatrix}
        0 & \rho\\
        0 & 0
    \end{bmatrix},
    \qquad
    \Omega_{i\leftarrow j_2}
    =
    \begin{bmatrix}
        0 & 0\\
        \rho & 0
    \end{bmatrix}.
\]
The corresponding \model{} transports are
\[
    T_{i\leftarrow j_1}
    =
    \frac{1}{2}
    \begin{bmatrix}
        1 & \rho\\
        0 & 1
    \end{bmatrix},
    \qquad
    T_{i\leftarrow j_2}
    =
    \frac{1}{2}
    \begin{bmatrix}
        1 & 0\\
        \rho & 1
    \end{bmatrix}.
\]

If a single linear map \(A\) applied after aggregation could reproduce both
edge-specific transports, \Cref{thm:postaggregation-equivalence}
would require
\[
    A
    =
    T_{i\leftarrow j_1}D_{i\leftarrow j_1}^{-1}
    =
    I_2+\Omega_{i\leftarrow j_1}
\]
and also
\[
    A
    =
    T_{i\leftarrow j_2}D_{i\leftarrow j_2}^{-1}
    =
    I_2+\Omega_{i\leftarrow j_2}.
\]
These two requirements are incompatible because
\(\Omega_{i\leftarrow j_1}\neq\Omega_{i\leftarrow j_2}\).

The same point can be seen directly. Let \(v=(1,0)^\top\) and compare the two
collections of incoming values
\[
    (v_{j_1},v_{j_2})=(v,0),
    \qquad
    (v_{j_1},v_{j_2})=(0,v).
\]
Standard attention maps both to the same aggregate, \(\frac{1}{2}v\).
Any linear map applied after aggregation therefore receives the same input in
the two cases and must return the same output.

With \model{}, the edge-specific maps are applied before the messages are
summed:
\[
    T_{i\leftarrow j_1}v
    =
    \frac{1}{2}
    \begin{bmatrix}
        1\\
        0
    \end{bmatrix},
    \qquad
    T_{i\leftarrow j_2}v
    =
    \frac{1}{2}
    \begin{bmatrix}
        1\\
        \rho
    \end{bmatrix}.
\]
Thus the two collections, which standard attention collapses to the same
aggregate, produce different outputs under edge-specific pre-aggregation
routing.

\subsection{Backbone-Specific Routing Instantiations}
\label{app:backbone-integration}

The routing construction in \Cref{sec:topa-transport} does not depend on
a particular attention scoring rule. Different backbones provide different
representations of the same interaction, which we use to construct the
edge-dependent routing matrix \(\Omega_e\).

\paragraph{GATv2.}
For GATv2 \citep{brody2022attentive}, let \(z_i\) denote the representation
entering the attention layer. We define the directed pre-attention context
\[
    u_{i\leftarrow j}
    =
    W_{\mathrm{recv}}z_i
    +
    W_{\mathrm{send}}z_j
    +
    W_E\mathbf{e}_{i\leftarrow j}
    +
    W_G\mathbf{g},
\]
where \(\mathbf{e}_{i\leftarrow j}\) and \(\mathbf{g}\) denote edge- and
graph-level features, respectively. We then set
\[
    \xi_{i\leftarrow j}
    =
    \operatorname{LeakyReLU}(u_{i\leftarrow j}),
\]
and use \(\xi_{i\leftarrow j}\) as the interaction representation in
\Cref{eq:topa-route-generator}. The resulting descriptor is mapped
jointly to the full \(H\times H\) routing matrix.

\paragraph{Graph Transformer.}
For the graph Transformer \citep{graphtransformer}, we first form the
edge-conditioned sender key
\[
    K_{i\leftarrow j}
    =
    K_j+E_{i\leftarrow j}+G,
\]
where \(K_j\) is the sender key and \(E_{i\leftarrow j}\) and \(G\) denote
edge- and graph-level embeddings in the same key space. Combining this with
the receiver query \(Q_i\), we define
\[
    \xi_{i\leftarrow j}
    =
    \operatorname{LeakyReLU}
    \left(
        \operatorname{vec}(Q_i+K_{i\leftarrow j})
    \right),
\]
which is supplied to the routing generator in
\Cref{eq:topa-route-generator}.

\subsection{Initialization and Computational Cost}
\label{app:initialization-complexity}

\paragraph{Initialization.}
For the routing generator, we use
\[
    g_\theta(\xi_e)
    =
    \operatorname{reshape}_{H\times H}
    \left(
        W_\Omega\xi_e+b_\Omega
    \right),
\]
with \(g_\theta:\mathbb R^p\to\mathbb R^{H\times H}\).
We initialize \(W_\Omega=0\) and \(b_\Omega=0\), so that
\(g_\theta(\xi_e)=0\) and therefore \(\Omega_e=0\) for every interaction.
\Cref{eq:topa-transport} then reduces to
\[
    m_e^{\model{}} = D_eV_j,
\]
which is exactly the message produced by the corresponding attention
baseline. With matched backbone parameters, \model{} therefore implements
the same function as the baseline at initialization.

This zero initialization does not prevent the routing generator from
learning. For \(\lambda>0\), \(\tanh'(0)=1\), so the parameters
\(W_\Omega\) and \(b_\Omega\) can receive first-order gradients at
initialization.

\paragraph{Computational cost.}
The additional cost comes from generating and applying a dense
\(H\times H\) routing matrix in head space. For the routing generator, an
interaction representation \(\xi_e\in\mathbb R^p\) requires \(O(pH^2)\)
operations to produce \(\Omega_e\), while applying the resulting matrix to
\(V_j\in\mathbb R^{H\times C}\) costs \(O(H^2C)\) per interaction.
Across the graph, the additional computation is therefore \(O\!\left(|\mathcal E|H^2(p+C)\right)\).
If the routing matrices are materialized explicitly, they require
\(O(|\mathcal E|H^2)\) additional memory.
\section{Additional Experimental Details, Results, and Analysis}
\label{app:additional_details}

This section complements the experimental evaluation in
\autoref{sec:exps} with additional dataset and implementation details,
hyperparameter configurations, extended results, and further analyses for
each experimental setting. All experiments were run on NVIDIA A100 SXM4
GPUs with 80~GB of memory. The code will be released upon acceptance.

\subsection{STaR Benchmark: Dataset, Task, and Model Configuration}
\label{app:star_description}

STaR \citep{khalid2025systematic} is a relational reasoning benchmark that
provides a controlled setting in which edge-dependent transformations are
themselves part of the computation. Each instance consists of an edge-labelled graph together with a query pair of nodes, and the task is to infer the relation holding between them. STaR contains two relational domains, each comprising 57{,}600 training instances and 153{,}600 test instances: RCC-8, with eight qualitative spatial relations, and Allen's Interval Algebra (IA), with thirteen temporal relations. In both domains, the relation carried by an edge determines not only whether information is relevant, but also how the current relational state should be transformed before propagation. Unlike conventional single-path reasoning benchmarks, STaR requires both composing relations along individual paths and integrating the partial, potentially disjunctive information provided
by multiple paths. Problem difficulty is controlled by the path length \(k\) and the number of paths \(b\) connecting the queried nodes. Models are trained on graphs with \(k\in\{2,3,4\}\) and \(b\in\{1,2,3\}\), and evaluated at \(k\in\{2,3,\ldots,15\}\) while keeping \(b\in\{1,2,3\}\). This allows us to evaluate systematic generalization to reasoning depths beyond those observed during training while keeping the number of paths within the training range.

Relation labels are supplied as edge information to both the attention backbone
and the \model{} routing mechanism, so that the comparison isolates the
additional capacity provided by cross-head transport rather than access to
relational information. Within each backbone family, the baseline and
\model{} variant share the same input representation, backbone architecture,
prediction head, optimization procedure, and tuning budget. As throughout our experiments, the routing branch is zero-initialized, so that each \model{} model starts from the same function as its corresponding attention baseline. Performance is measured by query-relation classification accuracy across the different \((k,b)\) configurations. We otherwise follow the experimental setting of \citet{khalid2025systematic}.



\textbf{Hyperparameters and training protocol.}
\label{sec:star-hyperparameters}
All STaR experiments use a fixed hyperparameter configuration.
Both GATv2 and the GT use a hidden dimension of 64,
four attention heads, and 15 recurrent message-passing rounds with
shared processor weights. The Top-A variant has a hidden
dimension of 64 and a transport scale of 1.
Training uses AdamW with learning rate $10^{-3}$, weight decay
$10^{-5}$, and batch size 128. Attention dropout is disabled,
while residual dropout is set to $0.1$.
Models train for at most 40 epochs, with early stopping after
eight epochs without improvement in validation accuracy;
the checkpoint with the highest validation accuracy is retained.

For both RCC-8 and Interval Algebra, the released training
mixture covers $b\in\{1,2,3\}$ and $k\in\{2,3,4\}$.
An 80/20 training--validation split is stratified jointly by
target relation and $(b,k)$ configuration.
Each experiment uses three training seeds.
Within each architecture, the baseline and Top-A variants share
the training configuration, initial shared parameters, minibatch
order, and dropout randomness. The baseline freezes the
off-diagonal controller at zero, whereas Top-A trains it from
the same zero initialization.

\subsection{MovieLens-100K: Experimental Setup and Analysis}
\label{app:movielens}

\noindent\textbf{Dataset, task, and model configuration.}
MovieLens-100K \citep{harper2015movielens} provides a real-world
heterogeneous setting in which information is exchanged between two
semantically distinct node types. We represent the dataset as a bipartite
graph containing 943 users, described by 24-dimensional demographic
features, and 1,682 movies, described by 18-dimensional genre features.
Each observed rating induces two directed interactions, \emph{rates}
(\(u\to m\)) and \emph{rated-by} (\(m\to u\)), which exchange the roles of
sender and receiver while connecting the same user--movie pair. We formulate
the task as binary user--movie link prediction. We use the official 80k/20k rating split for training and test positive
edges, respectively, and hold out 10\% of the training ratings as a fixed
validation set. Validation edges are removed from the message-passing graph.
Negative examples are sampled uniformly from unobserved user--movie pairs
using a 1:1 positive-to-negative ratio, with validation and test negatives
fixed across runs. Type-specific input encoders map user and movie attributes to a shared 64-dimensional
representation, to which we add a learned node-type embedding. Relation type
is represented by a one-hot edge feature and is available to both the
attention mechanism and the \model{} routing generator. We evaluate GATv2
and a sparse GT, together with their corresponding
\model{} variants. All architectures use \(H=4\) attention heads, two
message-passing layers, and a dot-product link-prediction decoder. For each backbone family, the baseline and \model{} variant share the same
input encoders, backbone architecture, decoder, and training protocol.
Hyperparameters are selected on validation performance using tuning seeds
that are disjoint from the final evaluation seeds. The learning-rate search
is shared across all models, while the \model{} variants additionally select
the transport scale \(\lambda\), as detailed below. Final results are evaluated over ten seeds that are disjoint from those used for hyperparameter selection, and test AUROC and AUPR are measured at the best-validation epoch. For each evaluation seed, the baseline and \model{} variant use the same initialization for all shared parameters, while the routing generator is zero-initialized. Consequently, the paired models implement the same function before optimization.

\paragraph{Hyperparameter selection and implementation details.}
We select the hyperparameters independently for each model using mean validation
AUROC over three tuning seeds ($0$--$2$), which are disjoint from the ten
evaluation seeds ($42$--$51$). The learning-rate grid is shared across all four models, while the route scale \(\lambda\) applies only to the \model{} variants.
All remaining architectural and optimization choices are fixed within each
backbone family and are reported in Table~\ref{tab:movielens-hpo-space}.
The configurations selected by validation performance are reported in
Table~\ref{tab:movielens-best-configs}. 
\begin{table}[h]
\centering
\small
\caption{MovieLens-100K hyperparameter search space, shared by GATv2,
GT, and their respective \model{} variants. }
\label{tab:movielens-hpo-space}
\begin{tabular}{ll}
\toprule
Hyperparameter & Search space \\
\midrule
Learning rate
    & $\{5\times10^{-4},\,10^{-3},\,2\times10^{-3}\}$ \\
Route scale $\lambda$ (\model{} only)
    & $\{0.5,\,1\}$ \\
\midrule
Hidden size
    & $64$ ($4$ heads $\times$ $16$ channels) \\
Attention heads
    & $4$ \\
Layers
    & $2$ \\
Dropout probability (input / feature / attention)
    & $0.2$ \\
Weight decay
    & $10^{-5}$ \\
Optimiser
    & Adam \\
Max.\ epochs / early-stopping patience
    & $1000$ / $100$ \\
Negative sampling
    & $1{:}1$ uniform, resampled every epoch \\
\bottomrule
\end{tabular}
\end{table}

\begin{table}[h]
\centering
\small
\caption{Selected configurations on MovieLens-100K. Each is the point of the grid in
Table~\ref{tab:movielens-hpo-space} with the highest mean validation
AUROC over the three tuning seeds.}
\label{tab:movielens-best-configs}
\begin{tabular}{llclrr}
\toprule
Model & Variant & Learning rate & $\lambda$ & Val.\ AUROC & Params \\
\midrule
GATv2 & --      & $10^{-3}$         & --    & $91.06$ & $28{,}416$ \\
      & \model{} & $2\times10^{-3}$ & $1$   & $92.46$ & $30{,}496$ \\
\midrule
GT    & --      & $2\times10^{-3}$  & --    & $79.34$ & $44{,}928$ \\
      & \model{} & $10^{-3}$        & $0.5$ & $92.51$ & $47{,}008$ \\
\bottomrule
\end{tabular}
\end{table}

\begin{figure}[t]
\centering
\begin{subfigure}[t]{0.24\linewidth}
  \centering
  \includegraphics[width=\linewidth]{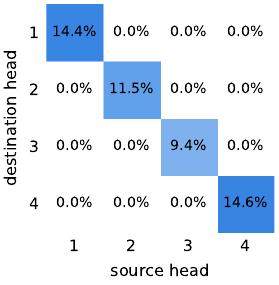}
  \caption{GATv2}\label{fig:movielens-asym-gatv2}
\end{subfigure}\hfill
\begin{subfigure}[t]{0.2\linewidth}
  \centering
  \includegraphics[width=\linewidth]{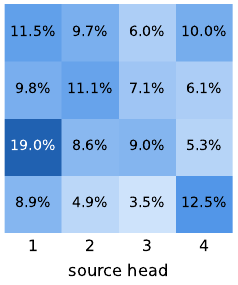}
  \caption{GATv2 + \model{}}\label{fig:movielens-asym-gatv2-topa}
\end{subfigure}\hfill
\begin{subfigure}[t]{0.2\linewidth}
  \centering
  \includegraphics[width=\linewidth]{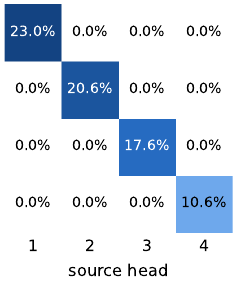}
  \caption{GT}\label{fig:movielens-asym-gt}
\end{subfigure}\hfill
\begin{subfigure}[t]{0.2\linewidth}
  \centering
  \includegraphics[width=\linewidth]{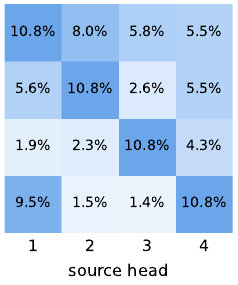}
  \caption{GT + \model{}}\label{fig:movielens-asym-gt-topa}
\end{subfigure}\hfill
\begin{minipage}[t]{0.075\linewidth}
  \raisebox{4mm}{\includegraphics[width=\linewidth]{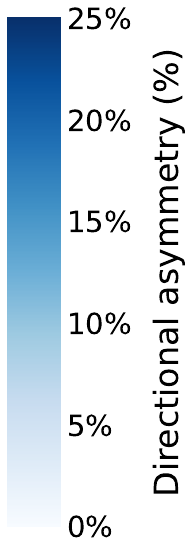}}
\end{minipage}
\caption{\emph{Relative Directional Asymmetry of the Mean Head-to-Head
Transport.} Entrywise normalized asymmetry \(N_{rs}\) from
\Cref{eq:movielens-directional-asymmetry}, multiplied by \(100\), between the mean effective transport for \emph{rates} (\(u\!\to\!m\)) and
\emph{rated-by} (\(m\!\to\!u\)). Rows index receiver heads and columns source heads. Standard GATv2 and GT have zero off-diagonal entries because vanilla attention does not permit cross-head transport. With \model{}, nonzero off-diagonal asymmetry indicates that the effective cross-head contributions differ between the two propagation directions.}
\label{fig:movielens-asym}
\end{figure}

\noindent\textbf{Directional asymmetry of the mean head-to-head transport.}
The paired relations \emph{rates} and \emph{rated-by} provide a natural
probe of whether the learned transport depends on interaction direction.
For every observed rating \((u,m)\in\mathcal R\), the graph contains the
directed interactions \(u\to m\) and \(m\to u\). These edges connect the
same pair of nodes but exchange their sender and receiver roles. For \model{}, the effective head-space transport associated with an
interaction \(e\) is:
\begin{equation}
    T_e
    =
    D_e\left(I_H+\lambda\Omega_e\right)
    \label{eq:movielens-effective-transport}
\end{equation}
while the corresponding baseline is recovered by setting
\(\Omega_e=0\), giving \(T_e=D_e\).
We average the effective transport separately over the two directions:
\begin{equation}
    \overline{T}_{\mathrm{rates}}
    =
    \frac{1}{|\mathcal R|}
    \sum_{(u,m)\in\mathcal R}
    T_{u\to m},
    \qquad
    \overline{T}_{\mathrm{rated\text{-}by}}
    =
    \frac{1}{|\mathcal R|}
    \sum_{(u,m)\in\mathcal R}
    T_{m\to u}
    \label{eq:movielens-mean-transport}
\end{equation}
and compare their entries through the normalized directional asymmetry:
\begin{equation}
    N_{rs}
    =
    \frac{
        \left|
        \overline{T}_{\mathrm{rates},rs}
        -
        \overline{T}_{\mathrm{rated\text{-}by},rs}
        \right|
    }{
        \left|
        \overline{T}_{\mathrm{rates},rs}
        \right|
        +
        \left|
        \overline{T}_{\mathrm{rated\text{-}by},rs}
        \right|
    }
    \in[0,1]
    \label{eq:movielens-directional-asymmetry}
\end{equation}
with \(N_{rs}:=0\) when both entries vanish. Here, \(r\) indexes the
receiver head and \(s\) the source head. The normalization makes the
comparison independent of a common rescaling of the two transports:
\(N_{rs}=0\) when the corresponding route is identical in both directions,
whereas \(N_{rs}=1\) when it is used in only one direction or the two
directional entries have opposite signs. 
\Cref{fig:movielens-asym} reports this quantity for the two backbones and
their \model{} variants. Standard attention contains no off-diagonal
transport, so directional differences are confined to the diagonal attention
paths. With \model{}, the effective transport also contains off-diagonal
contributions, whose directional asymmetry shows that cross-head information
transfer differs between \emph{rates} and \emph{rated-by}. The effective
head-space transport is therefore not identical across the two directed
relations.

\paragraph{Dependence of effective transport on movie content.}
We next examine whether the edge-specific effective transport varies only
with relation direction or also across individual interactions. For each
edge, we vectorize the \(H(H-1)=12\) off-diagonal entries of its effective
transport and project these vectors onto their first two principal
components. We then color the resulting edge-level representations according
to the genre of the movie endpoint. As shown in \Cref{fig:movielens-genre}, the effective cross-head transports
exhibit structure associated with movie genre. The organization is most
pronounced for \emph{rated-by} edges, where
Drama/Mystery/Thriller, Action/Adventure/Sci-Fi, and
Children's/Animation movies occupy different regions of the projection,
while a weaker structure is visible for \emph{rates}. Since genre enters
through the movie features rather than the link-prediction target, these
results show that the effective cross-head transport varies with node content and not only with the discrete relation direction. Together with the
directional analysis above, this shows that the transport realized by
\model{} varies across both relation direction and individual interactions.

\begin{figure}[t]
\centering
\begin{subfigure}[t]{0.478\linewidth}
  \centering
  \includegraphics[width=\linewidth]{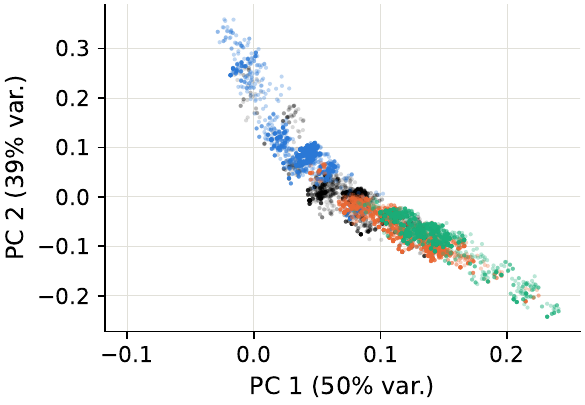}
  \caption{\emph{rated-by} (movie $\rightarrow$ user)}\label{fig:movielens-genre-rated-by}
\end{subfigure}\hfill
\begin{subfigure}[t]{0.402\linewidth}
  \centering
  \includegraphics[width=\linewidth]{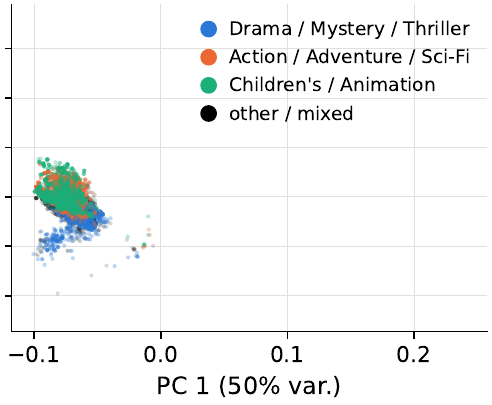}
  \caption{\emph{rates} (user $\rightarrow$ movie)}\label{fig:movielens-genre-rates}
\end{subfigure}
\caption{\emph{Per-edge Transport Reflects Movie Content.} PCA of the twelve off-diagonal transport coefficients for individual edges, colored by the genre of the movie endpoint. The learned transports separate according to genre, particularly for \emph{rated-by} edges, where distinct semantic groups occupy different regions of the projection. The effect is weaker but still visible for \emph{rates} edges. }
\label{fig:movielens-genre}
\end{figure}

\subsection{CLRS-30}\label{app:clrs}

\noindent\textbf{Setup.}
We evaluate on the CLRS Algorithmic Reasoning Benchmark
\citep{velickovic2022clrs}, which contains execution trajectories
from 30 classical algorithms drawn from the \emph{Introduction
to Algorithms} textbook \citep{cormen}. The benchmark spans a
broad range of computational procedures, including sorting,
searching, divide-and-conquer, greedy algorithms, dynamic
programming, graph algorithms, string matching, and computational
geometry. This diversity makes CLRS particularly suitable for
studying architectural inductive biases, since different tasks
require substantially different forms of computation, such as
repeated state updates, recursive reasoning, dynamic programming,
and structured information propagation.

Each algorithm is represented through a graph-oriented encoding
in which the objects manipulated by the algorithm are treated
as nodes and their relations as edges. Depending on the task,
features may be associated with nodes, edges, or the entire
graph. Each execution trajectory contains the algorithm input,
its final output, and a sequence of intermediate states,
referred to as \emph{hints}, which expose the evolution of the
underlying computation. These hints provide intermediate
supervision during training and are designed so that successive
algorithmic states can generally be related through local
message-passing operations.

We use the canonical \texttt{CLRS30\_v1.0.0} dataset and its predefined
training, validation, and test splits. Each algorithm has
1{,}000 training trajectories at nominal problem size $n=16$.
Validation uses the same nominal size, while testing uses
$n=64$ to assess out-of-distribution generalization, subject
to the benchmark's task-specific size conventions. Validation
and test splits contain 32 trajectories for most tasks, with
the benchmark's prescribed multipliers for tasks with fewer
prediction targets: 1{,}024 trajectories for maximum subarray
(Kadane), and 2{,}048 for binary search, minimum, quickselect,
both string-matching algorithms, and segment intersection.

We train a separate model for each algorithm, encoding and
decoding hints and supervising intermediate predictions.
Each final run uses 10{,}000 training steps, and the checkpoint
with the highest validation score is restored for test
evaluation. For optimal BST, we omit the auxiliary
\texttt{root\_h} hint loss in all variants.

\medskip
\noindent\textbf{Hyperparameter selection.}
We adapt the eight-core-task, reduced-budget tuning strategy
of \citet{diao2023relational}, using Optuna TPE and additional
multi-seed confirmation stages while retaining the canonical
CLRS-30 dataset splits. The core tasks are activity selection,
Bellman--Ford, binary search, maximum subarray (Kadane), Graham
scan, insertion sort, matrix-chain order, and naive string
matching.

The procedure is run independently for GATv2, GT, and their
respective \model{} variants, using identical search spaces,
training-step budgets, trial-seed schedules, and selection
criteria. \Cref{tab:clrs-hpo-space} lists the search space.
For each variant, we first evaluate 20 configurations on the
eight core tasks, using one seed per configuration and
2{,}000 training steps per task. The three highest-ranked
configurations are then evaluated on three confirmation
seeds at the same budget. The best two proceed to confirmation
at 10{,}000 steps, using the same three seeds.

At each stage, configurations are ranked by the mean
best-validation score across the core tasks and applicable
seeds. The winning configuration is used for all 30 algorithms,
with 20 paired final-evaluation seeds disjoint from the search
and confirmation seeds. All selected configurations use hidden
size 256, eight attention heads, one message-passing step per
hint transition, batch size 32, and dropout probability 0.1.
These comparisons therefore assess separately tuned systems.


\begin{table*}[t]
\centering
\small
\setlength{\tabcolsep}{4pt}
\caption{CLRS-30 hyperparameter search space and final selected configurations
for GATv2, GT, and their respective \model{} variants. Each variant is tuned
independently using 20 Optuna TPE trials on eight core algorithms, followed by
multi-seed and full-budget confirmation. Final configurations are selected using
validation performance averaged over the eight core tasks and three confirmation
seeds, then frozen for final evaluation. Learning rate is searched on a logarithmic
scale; all other parameters use categorical choices. Selected learning rates are
rounded to four significant figures. Message-passing steps are per hint transition.}
\label{tab:clrs-hpo-space}
\resizebox{\linewidth}{!}{%
\begin{tabular}{llrrrr}
\toprule
& & \multicolumn{4}{c}{Final selected configuration} \\
\cmidrule(lr){3-6}
Hyperparameter & Search space & GATv2 & GATv2 + \model{} & GT & GT + \model{} \\
\midrule
Learning rate & $[10^{-4},\,3\times10^{-3}]$ (log scale) & $4.727\times10^{-4}$ & $4.563\times10^{-4}$ & $7.348\times10^{-4}$ & $5.416\times10^{-4}$ \\
Hidden size & $\{64,128,256\}$ & 256 & 256 & 256 & 256 \\
Attention heads & $\{2,4,8\}$ & 8 & 8 & 8 & 8 \\
Message-passing steps & $\{1,2\}$ & 1 & 1 & 1 & 1 \\
Batch size & $\{16,32\}$ & 32 & 32 & 32 & 32 \\
Weight decay & $\{0,10^{-4},10^{-3}\}$ & $10^{-3}$ & $10^{-4}$ & $10^{-4}$ & $10^{-4}$ \\
Dropout probability & $\{0,0.1,0.2\}$ & 0.1 & 0.1 & 0.1 & 0.1 \\
\bottomrule
\end{tabular}%
}
\end{table*}

\medskip
\noindent\textbf{Additional results.}
\Cref{tab:clrs-tuned} reports per-task CLRS-30 test performance
for GATv2, GATv2 + \model{}, GT, and GT + \model{}, with all
30 tasks evaluated across 20 paired seeds. Entries report
the mean and sample standard deviation, the paired mean
improvement in percentage points, and the number of seeds
on which adding \model{} strictly improves the score.
The macro-average weights all 30 tasks equally; its standard
deviation and win count are computed from seed-level task
averages.

GATv2 + \model{} increases the macro-average from 53.22\% to
59.26\% (+6.04 percentage points), while GT + \model{} increases
it from 49.81\% to 58.36\% (+8.56 percentage points). For both
architectures, the task-averaged score improves in all 20 seeds.

Both architectures show particularly pronounced gains on binary
search, insertion sort, bubble sort, and quicksort. GT also
benefits substantially on shortest-path tasks, including Dijkstra
(+21.50 percentage points), Bellman--Ford (+13.77), Floyd--Warshall
(+12.23), and DAG shortest paths (+10.90). Improvements are not
uniform: heapsort deteriorates for both architectures, and optimal
BST exhibits the largest mean decline for GT (-15.93 percentage
points).

\providecommand{\model}{\textsc{OffDiag}}
\begin{table*}[h]
\centering
\footnotesize
\setlength{\tabcolsep}{3pt}
\caption{CLRS-30 test performance (official CLRS score) after separate hyperparameter optimization for each architecture and variant. Entries are mean $\pm$ sample standard deviation across 20 paired seeds, in percent (higher is better). $\Delta$ is the improvement from adding \model{} in percentage points, computed before rounding. Wins counts seeds with strictly higher scores with \model{} than without it (ties are not wins). Bold indicates the higher mean within each architecture, with ties at displayed precision unbolded; it does not denote statistical significance. All evaluations are complete. The macro-average gives equal weight to all 30 tasks; its standard deviation is computed across seed-level task averages, and its wins compare those averages across the 20 seeds. Optimal BST uses the auxiliary root-hint loss workaround described in the experimental setup.}
\label{tab:clrs-tuned}
\resizebox{\linewidth}{!}{%
\begin{tabular}{lrrrcrrrc}
\toprule
& \multicolumn{4}{c}{GATv2} & \multicolumn{4}{c}{GT} \\
\cmidrule(lr){2-5} \cmidrule(lr){6-9}
Task & GATv2 & GATv2 + \model{} & $\Delta$ & Wins & GT & GT + \model{} & $\Delta$ & Wins \\
\midrule
Activity selector & $\mathbf{67.88} \pm 4.44$ & $65.29 \pm 5.16$ & $-2.59$ & 7/20 & $\mathbf{69.19} \pm 4.48$ & $66.63 \pm 5.85$ & $-2.56$ & 6/20 \\
Articulation points & $\mathbf{28.88} \pm 5.21$ & $24.24 \pm 10.35$ & $-4.64$ & 4/20 & $\mathbf{24.56} \pm 11.04$ & $22.27 \pm 12.01$ & $-2.29$ & 10/20 \\
Bellman--Ford & $88.81 \pm 1.18$ & $\mathbf{90.38} \pm 2.36$ & $+1.57$ & 13/20 & $78.23 \pm 1.61$ & $\mathbf{92.00} \pm 1.12$ & $+13.77$ & 20/20 \\
BFS & $\mathbf{99.99} \pm 0.02$ & $99.98 \pm 0.02$ & $-0.00$ & 1/20 & $\mathbf{99.70} \pm 0.75$ & $97.99 \pm 2.07$ & $-1.71$ & 3/20 \\
Binary search & $40.59 \pm 12.98$ & $\mathbf{80.39} \pm 8.00$ & $+39.80$ & 20/20 & $5.60 \pm 0.89$ & $\mathbf{79.91} \pm 8.09$ & $+74.32$ & 20/20 \\
Bridges & $32.91 \pm 5.67$ & $\mathbf{34.46} \pm 7.74$ & $+1.56$ & 12/20 & $29.94 \pm 0.60$ & $\mathbf{30.00} \pm 0.43$ & $+0.06$ & 10/20 \\
Bubble sort & $10.41 \pm 16.92$ & $\mathbf{54.17} \pm 16.83$ & $+43.76$ & 18/20 & $5.98 \pm 5.09$ & $\mathbf{51.53} \pm 23.64$ & $+45.55$ & 20/20 \\
DAG shortest paths & $94.63 \pm 2.10$ & $\mathbf{96.59} \pm 0.89$ & $+1.96$ & 17/20 & $82.14 \pm 4.15$ & $\mathbf{93.03} \pm 4.65$ & $+10.90$ & 19/20 \\
DFS & $21.70 \pm 4.41$ & $\mathbf{27.44} \pm 6.57$ & $+5.74$ & 15/20 & $18.93 \pm 5.80$ & $\mathbf{25.45} \pm 6.34$ & $+6.52$ & 14/20 \\
Dijkstra & $80.20 \pm 5.25$ & $\mathbf{88.49} \pm 2.64$ & $+8.29$ & 20/20 & $69.11 \pm 7.51$ & $\mathbf{90.61} \pm 2.72$ & $+21.50$ & 20/20 \\
Maximum subarray (Kadane) & $58.10 \pm 4.78$ & $\mathbf{63.61} \pm 3.34$ & $+5.51$ & 16/20 & $57.19 \pm 4.66$ & $\mathbf{62.91} \pm 3.73$ & $+5.72$ & 17/20 \\
Floyd--Warshall & $\mathbf{34.68} \pm 2.86$ & $31.73 \pm 4.16$ & $-2.94$ & 5/20 & $27.62 \pm 5.89$ & $\mathbf{39.85} \pm 3.59$ & $+12.23$ & 20/20 \\
Graham scan & $71.40 \pm 5.07$ & $\mathbf{78.61} \pm 6.13$ & $+7.20$ & 17/20 & $66.70 \pm 8.12$ & $\mathbf{72.96} \pm 11.18$ & $+6.26$ & 14/20 \\
Heapsort & $\mathbf{54.47} \pm 25.45$ & $28.83 \pm 14.37$ & $-25.64$ & 4/20 & $\mathbf{45.04} \pm 20.21$ & $33.43 \pm 21.74$ & $-11.61$ & 7/20 \\
Insertion sort & $28.49 \pm 20.03$ & $\mathbf{84.49} \pm 5.55$ & $+56.00$ & 19/20 & $37.24 \pm 21.72$ & $\mathbf{87.33} \pm 3.86$ & $+50.10$ & 19/20 \\
Jarvis march & $58.95 \pm 2.30$ & $\mathbf{60.08} \pm 2.25$ & $+1.13$ & 13/20 & $61.96 \pm 4.88$ & $\mathbf{62.27} \pm 2.32$ & $+0.32$ & 10/20 \\
KMP matcher & $0.00 \pm 0.01$ & $\mathbf{0.05} \pm 0.12$ & $+0.05$ & 4/20 & $0.05 \pm 0.10$ & $\mathbf{0.78} \pm 1.14$ & $+0.73$ & 16/20 \\
LCS length & $84.81 \pm 0.97$ & $\mathbf{85.21} \pm 0.39$ & $+0.40$ & 13/20 & $\mathbf{85.33} \pm 0.63$ & $84.86 \pm 1.42$ & $-0.48$ & 9/20 \\
Matrix chain order & $77.98 \pm 6.92$ & $\mathbf{83.50} \pm 1.58$ & $+5.51$ & 17/20 & $\mathbf{86.48} \pm 5.30$ & $83.92 \pm 1.48$ & $-2.56$ & 2/20 \\
Minimum & $\mathbf{86.36} \pm 4.84$ & $83.39 \pm 8.26$ & $-2.97$ & 8/20 & $\mathbf{90.16} \pm 3.80$ & $88.48 \pm 3.61$ & $-1.67$ & 5/20 \\
MST (Kruskal) & $\mathbf{64.00} \pm 6.38$ & $56.98 \pm 10.83$ & $-7.02$ & 5/20 & $\mathbf{73.12} \pm 4.03$ & $69.57 \pm 4.91$ & $-3.56$ & 6/20 \\
MST (Prim) & $59.70 \pm 7.71$ & $\mathbf{66.28} \pm 7.23$ & $+6.58$ & 13/20 & $52.26 \pm 8.63$ & $\mathbf{60.76} \pm 8.21$ & $+8.51$ & 15/20 \\
Naive string matcher & $2.79 \pm 1.70$ & $\mathbf{3.03} \pm 1.53$ & $+0.25$ & 11/20 & $2.74 \pm 1.60$ & $\mathbf{3.06} \pm 1.22$ & $+0.32$ & 12/20 \\
Optimal BST & $\mathbf{57.83} \pm 9.23$ & $52.42 \pm 17.87$ & $-5.41$ & 7/20 & $\mathbf{57.21} \pm 12.46$ & $41.28 \pm 21.51$ & $-15.93$ & 8/20 \\
Quickselect & $0.00 \pm 0.00$ & $\mathbf{0.40} \pm 1.18$ & $+0.40$ & 5/20 & $0.01 \pm 0.02$ & $\mathbf{0.33} \pm 1.50$ & $+0.33$ & 1/20 \\
Quicksort & $7.36 \pm 5.14$ & $\mathbf{50.20} \pm 14.13$ & $+42.84$ & 20/20 & $6.99 \pm 6.62$ & $\mathbf{44.39} \pm 25.57$ & $+37.40$ & 18/20 \\
Segments intersect & $\mathbf{80.43} \pm 0.67$ & $79.19 \pm 0.79$ & $-1.24$ & 4/20 & $\mathbf{81.47} \pm 0.38$ & $79.53 \pm 0.74$ & $-1.94$ & 0/20 \\
Strongly connected comp. & $32.68 \pm 4.80$ & $\mathbf{34.22} \pm 7.24$ & $+1.54$ & 15/20 & $35.42 \pm 3.53$ & $\mathbf{38.54} \pm 6.56$ & $+3.11$ & 14/20 \\
Task scheduling & $80.57 \pm 1.30$ & $\mathbf{82.55} \pm 1.15$ & $+1.98$ & 17/20 & $80.76 \pm 1.04$ & $\mathbf{81.25} \pm 1.53$ & $+0.49$ & 13/20 \\
Topological sort & $90.03 \pm 12.34$ & $\mathbf{91.55} \pm 7.88$ & $+1.52$ & 11/20 & $63.13 \pm 4.77$ & $\mathbf{66.00} \pm 17.63$ & $+2.87$ & 13/20 \\
\midrule
\textbf{Macro-average (30 tasks)} & $53.22 \pm 1.26$ & $\mathbf{59.26} \pm 1.48$ & $+6.04$ & 20/20 & $49.81 \pm 1.73$ & $\mathbf{58.36} \pm 1.20$ & $+8.56$ & 20/20 \\
\bottomrule
\end{tabular}%
}
\end{table*}

\subsection{Heterophilic Node Classification}
\label{app:platonov}
Heterophily motivated much of the early work on SNNs: when neighboring nodes
tend to carry different labels, transforming a neighboring representation
before aggregation can be more expressive than merely reweighting it
\citep{bodnar2022neural}. Since \model{} implements a matrix-valued transport mechanism within multi-head attention, the benchmarks of
\citet{platonov2023critical} provide a natural contrast setting in which to ask whether cross-head transport provides a consistent advantage on heterophilic graphs. This question is also relevant in light of recent work questioning whether the empirical advantages of sheaf models on heterophilic graphs can be attributed specifically to learned sheaf structure \citep{caralt2026necessity,fiorini2026benchmarkingsheafneuralnetworks}.

\paragraph{Setup.}
We evaluate on the seven benchmarks of \citet{platonov2023critical}, using
the ten official train/validation/test splits provided with each dataset.
The multi-class datasets (\emph{roman-empire}, \emph{amazon-ratings},
\emph{squirrel}, and \emph{chameleon}) are evaluated by accuracy, while the
binary datasets (\emph{minesweeper}, \emph{tolokers}, and \emph{questions})
are evaluated by ROC-AUC. Early stopping and hyperparameter selection use the corresponding validation metric. Within each backbone family, the baseline and its \model{} variant share the same encoder, processor depth, decoder, optimization schedule, data split, and seed. Each \model{} run is initialized from the corresponding same-seed baseline, while the routing generator is zero-initialized, so that the two models implement the same function before training. All models use \(H=4\) attention heads and five message-passing layers, with head width \(64\) for GATv2 and \(56\) for the GT. Parameter counts for all models are reported in \Cref{tab:platonov-configs}. Each arm is given the same hyperparameter search over learning rate \(\in\{10^{-3},3\cdot10^{-4}\}\) and dropout \(\in\{0,0.2,0.5\}\), tuned on split \(0\) using two seeds and selected by validation performance alone. Before training, every \model{} run is verified to match its paired baseline at \(\Omega_e:=0\).

\begin{table*}[t]
    \centering
    \small
    \setlength{\tabcolsep}{5pt}
    \renewcommand{\arraystretch}{1.05}

    \caption{
    Performance comparison on the datasets of
    \citet{platonov2023critical}. AUROC is reported for
    \emph{minesweeper}, \emph{tolokers}, and \emph{questions}, while
    accuracy is reported for the remaining datasets. We additionally report
    GATv2 and GT with and without \model{}.
    }
    \label{tab:platonov_results}

    \resizebox{\textwidth}{!}{%
    \begin{tabular}{lccccccc}
        \toprule
        \textbf{Model}
        & \textbf{roman-empire}
        & \textbf{amazon-ratings}
        & \textbf{minesweeper}
        & \textbf{tolokers}
        & \textbf{questions}
        & \textbf{squirrel}
        & \textbf{chameleon} \\

        \textbf{Edge Homophily}
        & 0.05
        & 0.38
        & 0.68
        & 0.59
        & 0.84
        & 0.20
        & 0.23 \\
        \midrule

        \textit{Baselines}\\
        GCN
        & $73.69 \pm 0.74$
        & $48.70 \pm 0.63$
        & $89.75 \pm 0.52$
        & $83.64 \pm 0.67$
        & $76.09 \pm 1.27$
        & \textcolor{gray}{\textbf{$39.47 \pm 1.47$}}
        & $40.89 \pm 4.12$ \\

        SAGE
        & $85.74 \pm 0.67$
        & $53.63 \pm 0.39$
        & $93.51 \pm 0.57$
        & $82.43 \pm 0.44$
        & $76.44 \pm 0.62$
        & $36.09 \pm 1.99$
        & $37.77 \pm 4.14$ \\

        GAT
        & $80.87 \pm 0.30$
        & $49.09 \pm 0.63$
        & $92.01 \pm 0.68$
        & $83.70 \pm 0.47$
        & $77.43 \pm 1.20$
        & $35.62 \pm 2.06$
        & $39.21 \pm 3.08$ \\

        GAT-sep
        & $88.75 \pm 0.41$
        & $52.70 \pm 0.62$
        & $93.91 \pm 0.35$
        & $83.78 \pm 0.43$
        & $76.79 \pm 0.71$
        & $35.46 \pm 3.10$
        & $39.26 \pm 2.50$ \\

        GT
        & $86.51 \pm 0.73$
        & $51.17 \pm 0.66$
        & $91.85 \pm 0.76$
        & $83.23 \pm 0.64$
        & $77.95 \pm 0.68$
        & $36.30 \pm 1.98$
        & $38.87 \pm 3.66$ \\

        GT-sep
        & $87.32 \pm 0.39$
        & $52.18 \pm 0.80$
        & $92.29 \pm 0.47$
        & $82.52 \pm 0.92$
        & $78.05 \pm 0.93$
        & $36.66 \pm 1.63$
        & $40.31 \pm 3.01$ \\

        CO-GNN
        & $89.44 \pm 0.50$
        & \textbf{$54.20 \pm 0.34$}
        & $97.35 \pm 0.63$
        & \textcolor{gray}{\textbf{$84.84 \pm 0.96$}}
        & $75.97 \pm 0.89$
        & $39.39 \pm 2.76$
        & \textcolor{gray}{\textbf{$41.14 \pm 5.40$}} \\

        \midrule
        \textit{Sheaf Neural Networks}\\

        O(d)-NSD
        & $80.41 \pm 0.72$
        & $42.76 \pm 0.54$
        & $92.15 \pm 0.84$
        & $78.83 \pm 0.76$
        & $69.69 \pm 1.46$
        & $35.79 \pm 3.34$
        & $37.93 \pm 2.24$ \\

        BuNN
        & $91.75 \pm 0.39$
        & $53.74 \pm 0.51$
        & $98.99 \pm 0.16$
        & $84.78 \pm 0.80$
        & $78.75 \pm 1.09$
        & --
        & -- \\

        CSNN
        & $92.63 \pm 0.50$
        & $52.07 \pm 1.00$
        & $99.07 \pm 0.25$
        & $85.45 \pm 0.53$
        & $79.31 \pm 1.22$
        & $41.18 \pm 2.23$
        & $43.09 \pm 3.17$ \\

        \midrule

        \textit{\modelext{} Baselines and Models}\\

        GATv2 
        & $87.64 \pm 0.56$
        & $44.16 \pm 0.52$
        & $94.90 \pm 0.35$
        & $84.82 \pm 0.76$
        & $59.51 \pm 6.51$
        & $37.68 \pm 2.53$
        & $40.00 \pm 3.91$ \\

        \textbf{GATv2 + \model{} (ours)}
        & $87.81 \pm 0.41$
        & $45.23 \pm 1.07$
        & $94.96 \pm 0.48$
        & $85.09 \pm 0.75$
        & $58.59 \pm 7.03$
        & $36.96 \pm 2.12$
        & $40.09 \pm 3.53$ \\

        GT 
        & $88.10 \pm 0.40$
        & $52.75 \pm 0.56$
        & $94.01 \pm 0.32$
        & $77.89 \pm 7.76$
        & $60.08 \pm 3.45$
        & $38.17 \pm 2.27$
        & $41.19 \pm 3.32$ \\

        \textbf{GT + \model{} (ours)}
        & $89.00 \pm 0.53$
        & $52.69 \pm 0.67$
        & $93.97 \pm 0.29$
        & $84.76 \pm 0.67$
        & $60.57 \pm 4.49$
        & $38.39 \pm 2.33$
        & $41.11 \pm 3.83$ \\

        \bottomrule
    \end{tabular}%
    }
\end{table*}

\paragraph{Results.}
As shown in \Cref{tab:platonov_results}, cross-head routing does not provide a systematic advantage on these heterophilic benchmarks. Across the \(14\) dataset--backbone comparisons, most changes are small and vary in sign across datasets and architectures. The clearest exception is the GT on
\emph{tolokers}, where \model{} improves ROC-AUC from \(77.89\) to \(84.76\). A smaller positive effect is also observed for the GT on
\emph{roman-empire} (\(+0.90\) accuracy points), while the remaining
comparisons remain close to their corresponding baselines. Overall, adding
\model{} does not consistently improve performance over the corresponding
attention backbone in this benchmark suite.

\paragraph{What heterophily does and does not imply.}
These results help clarify the role of the transport inductive bias that
motivates \model{}. Heterophily was one of the original settings in which
matrix-valued sheaf transport was argued to be useful, since neighboring
representations may need to be transformed rather than simply reweighted.
By isolating a related matrix-valued transport mechanism within attention,
we can ask whether this additional flexibility is itself sufficient to yield
systematic gains on heterophilic graphs. The results suggest that it is not:
cross-head transport provides no consistent advantage across these benchmarks. Together with the main experiments, this supports a more specific view of the inductive bias, whose usefulness appears most clearly in settings where individual interactions call for different transformations before aggregation.

\paragraph{Hyperparameters and Selected Configurations}
\label{app:platonov-configs}
\Cref{tab:platonov-hparams} reports the complete configuration used for
heterophilic node classification. Only the learning rate and dropout are
searched, using the same six-cell grid for every arm; all remaining settings
are fixed in advance. Hyperparameters are selected according to validation
performance, following the standard protocol for these benchmarks.
\Cref{tab:platonov-configs} reports the selected configuration for each
dataset and model together with its validation score and parameter count.

\begin{table}[t]
\centering
\caption{\emph{Heterophilic Node Classification: Hyperparameters.}
The same hyperparameter grid is assigned to every arm and selected on
validation performance.}
\label{tab:platonov-hparams}
\scalebox{0.7}{%
\begin{tabular}{lcc}
\toprule
\textbf{Hyperparameter} & \textbf{GATv2} & \textbf{GT} \\
\midrule
\multicolumn{3}{l}{\emph{Tuned per arm (identical grid, validation-selected)}} \\
Learning rate & \multicolumn{2}{c}{$10^{-3}$ / $3\cdot 10^{-4}$} \\
Dropout ($=$ input dropout) & \multicolumn{2}{c}{$0$ / $0.2$ / $0.5$} \\
\midrule
\multicolumn{3}{l}{\emph{Architecture}} \\
Attention heads $H$ & $4$ & $4$ \\
Layers & $5$ & $5$ \\
Channels per head & $64$ & $56$ \\
Residual connections & yes & yes \\
LayerNorm in the block & -- & yes \\
Activation & -- & ReLU \\
LeakyReLU negative slope & $0.2$ & -- \\
Attention dropout & $0$ & -- \\
Output $\ell_2$ normalisation & yes & yes \\
Self-loops added & yes & yes \\
Jumping knowledge & no & no \\
\midrule
\multicolumn{3}{l}{\emph{Optimisation}} \\
Optimiser & Adam & Adam \\
Weight decay & $5\cdot 10^{-4}$ / $10^{-3}$ & $5\cdot 10^{-4}$ / $10^{-3}$ \\
Max.\ epochs & $1000$ & $1000$ \\
Early-stopping patience & $200$ & $200$ \\
\midrule
\multicolumn{3}{l}{\emph{\model{} transport}} \\
Route scale $\lambda$ & $1$ & $1$ \\
Route dropout & $0$ & $0$ \\
Routing initialization & zero & zero \\
\bottomrule
\end{tabular}}
\end{table}

\begin{table}[t]
\centering
\caption{\emph{Heterophilic Node Classification: Selected Configurations.}
Learning rate and dropout chosen for each arm by mean validation score from
Table~\ref{tab:platonov-hparams}. A dash in the variant column marks the plain
backbone; \model{} denotes the same backbone augmented with cross-head routing.}
\label{tab:platonov-configs}
\scalebox{0.7}{%
\begin{tabular}{lllccrr}
\toprule
\textbf{Dataset} & \textbf{Model} & \textbf{Variant} & \textbf{Learning rate} & \textbf{Dropout} & \textbf{Val.} & \textbf{Params} \\
\midrule
Roman-empire & GATv2 & -- & $0.001$ & $0.2$ & $88.47$ & $1,069,842$ \\
 & GATv2 & \model{} & $0.001$ & $0.5$ & $88.17$ & $1,080,162$ \\
 & GT & -- & $0.001$ & $0.5$ & $88.59$ & $1,081,714$ \\
 & GT & \model{} & $0.001$ & $0.5$ & $89.55$ & $1,099,714$ \\
\midrule
Amazon-ratings & GATv2 & -- & $0.0003$ & $0.2$ & $44.02$ & $1,066,501$ \\
 & GATv2 & \model{} & $0.0003$ & $0.2$ & $45.12$ & $1,076,821$ \\
 & GT & -- & $0.0003$ & $0.2$ & $52.43$ & $1,078,789$ \\
 & GT & \model{} & $0.0003$ & $0.2$ & $52.29$ & $1,096,789$ \\
\midrule
Squirrel-filt. & GATv2 & -- & $0.001$ & $0$ & $39.35$ & $1,524,485$ \\
 & GATv2 & \model{} & $0.001$ & $0.2$ & $39.14$ & $1,534,805$ \\
 & GT & -- & $0.001$ & $0$ & $39.83$ & $1,479,525$ \\
 & GT & \model{} & $0.001$ & $0.2$ & $40.32$ & $1,497,525$ \\
\midrule
Chameleon-filt. & GATv2 & -- & $0.001$ & $0.5$ & $47.04$ & $1,584,901$ \\
 & GATv2 & \model{} & $0.001$ & $0.5$ & $47.21$ & $1,595,221$ \\
 & GT & -- & $0.0003$ & $0.5$ & $47.04$ & $1,532,389$ \\
 & GT & \model{} & $0.0003$ & $0.5$ & $46.86$ & $1,550,389$ \\
\midrule
Minesweeper & GATv2 & -- & $0.001$ & $0.2$ & $94.02$ & $990,722$ \\
 & GATv2 & \model{} & $0.001$ & $0.2$ & $94.05$ & $1,001,042$ \\
 & GT & -- & $0.0003$ & $0.2$ & $93.26$ & $1,012,482$ \\
 & GT & \model{} & $0.0003$ & $0.2$ & $93.17$ & $1,030,482$ \\
\midrule
Tolokers & GATv2 & -- & $0.001$ & $0$ & $85.05$ & $991,490$ \\
 & GATv2 & \model{} & $0.001$ & $0$ & $84.97$ & $1,012,050$ \\
 & GT & -- & $0.001$ & $0.2$ & $85.35$ & $1,013,154$ \\
 & GT & \model{} & $0.0003$ & $0.2$ & $84.60$ & $1,031,154$ \\
\midrule
Questions & GATv2 & -- & $0.001$ & $0.5$ & $64.72$ & $1,065,986$ \\
 & GATv2 & \model{} & $0.001$ & $0.5$ & $64.87$ & $1,076,306$ \\
 & GT & -- & $0.001$ & $0.2$ & $63.26$ & $1,078,338$ \\
 & GT & \model{} & $0.001$ & $0.2$ & $64.79$ & $1,096,338$ \\
\bottomrule
\end{tabular}}
\end{table}



\end{document}